\documentclass{article}
\usepackage{microtype}
\usepackage{graphicx}
\usepackage{subfigure}
\usepackage{tikz}
\usepackage{lipsum}
\usepackage{adjustbox}
\usepackage{booktabs} % for professional tables
\usepackage{pgf}
\usepackage{multirow}
\usepackage{threeparttable}
\usepackage{float}
\usepackage{natbib}
\usepackage{hyperref}
\usepackage{amsmath}
\usepackage{amssymb}
\usepackage{mathtools}
\usepackage{amsthm}
\usepackage{tikz-cd}
\usepackage{algorithm}
\usepackage{algorithmic}
\usetikzlibrary{positioning, calc, decorations.pathreplacing}
\newtheorem{theorem}{Theorem}
\newtheorem{proposition}[theorem]{Proposition}

\newtheorem{definition}{Definition}

\newtheorem{remark}{Remark}
\newtheorem{example}{Example}[section]
\usepackage{mathtools, cuted}

\newcommand{\indep}{\perp\!\!\!\perp}
\newcommand{\E}{\mathbb{E}}

\usepackage{cuted}
\usepackage[preprint]{icml2026}

\usepackage[capitalize,noabbrev]{cleveref}

\usepackage[textsize=tiny]{todonotes}

\icmltitlerunning{Perturbative methods for non-parametric instrumental variable}

\begin{document}

\twocolumn[
  \icmltitle{Perturbative methods for non-parametric instrumental variable}

  % It is OKAY to include author information, even for blind submissions: the
  % style file will automatically remove it for you unless you've provided
  % the [accepted] option to the icml2026 package.

  % List of affiliations: The first argument should be a (short) identifier you
  % will use later to specify author affiliations Academic affiliations
  % should list Department, University, City, Region, Country Industry
  % affiliations should list Company, City, Region, Country

  % You can specify symbols, otherwise they are numbered in order. Ideally, you
  % should not use this facility. Affiliations will be numbered in order of
  % appearance and this is the preferred way.
  \icmlsetsymbol{equal}{*}

  \begin{icmlauthorlist}
    \icmlauthor{Wei Bu}{yyy}
    \icmlauthor{Arthur Gretton}{comp1}
  \end{icmlauthorlist}

  \icmlaffiliation{yyy}{Harvard University, Northeastern University, Physics Department}
  % \icmlaffiliation{xxx}{Northeastern University, Physics Department}
  \icmlaffiliation{comp1}{University College London}
  % \icmlaffiliation{comp2}{Google Deepmind}
  
  \icmlcorrespondingauthor{Wei Bu}{wbu112358@gmail.com}
  \icmlcorrespondingauthor{Arthur Gretton}{arthur.gretton@gmail.com}

  % You may provide any keywords that you find helpful for describing your
  % paper; these are used to populate the "keywords" metadata in the PDF but
  % will not be shown in the document
  \icmlkeywords{Machine Learning, ICML}

  \vskip 0.3in
]

% this must go after the closing bracket ] following \twocolumn[ ...

% This command actually creates the footnote in the first column listing the
% affiliations and the copyright notice. The command takes one argument, which
% is text to display at the start of the footnote. The \icmlEqualContribution
% command is standard text for equal contribution. Remove it (just {}) if you
% do not need this facility.

% Use ONE of the following lines. DO NOT remove the command.
% If you have no special notice, KEEP empty braces:
\printAffiliationsAndNotice{}  % no special notice (required even if empty)
% Or, if applicable, use the standard equal contribution text:
% \printAffiliationsAndNotice{\icmlEqualContribution}

\begin{abstract}
We introduce a perturbative approach for nonparametric instrumental variable (NPIV) estimation. By drawing inspiration from perturbation theory in physics, we extend standard kernel ridge methods with systematic higher perturbation order corrections that significantly improve estimation accuracy. Spectrally, the perturbation introduces mixing between different eigenmodes of the expectation integral operator, which becomes especially useful when the integral equation is ill-defined. One source for such ill-definedness can be the curse of dimensionality. Our method performs across various dimensionality regimes, particularly when the dimensionality parameter $\beta$ which is defined through the number of samples $n$ and dimension $d$ as $n^\beta = d$, becomes large. Experimental results show that our first-order perturbative corrections can reduce prediction error by up to 99\% in high-dimensional ill-defined cases ($\beta > 0.7$) compared to standard ridge regression approaches. The performance improvement is maintained across a wide range of dimensions, with the advantage becoming more pronounced as dimensionality increases. 
\end{abstract}

\section{Introduction}
Nonparametric instrumental variable (NPIV) estimation has emerged as a fundamental tool for causal inference in the presence of unmeasured confounding. The method leverages instrumental variables--variables that affect the treatment but not the outcome directly--to identify causal effects without imposing restrictive parametric assumptions on the functional form of the causal relationship. Despite its theoretical appeal, NPIV estimation remains a challenging problem especially when condition number of the expectation integral operator is high. This can be the result of multiple different factors: the presence of weak instrumental, misalignment between target and kernel spectrum and curse of dimensionality. Informatively speaking, we consider a setting to be "high dimensional" when sample size is small compared to the dimensionality of the variables. As the dimension of the feature space increases relative to sample size, estimation accuracy typically deteriorates. 

% Traditional approaches to NPIV often struggle when the dimensionality parameter $\beta$ exceeds certain thresholds. 

Classical approaches to NPIV estimation typically formulate the problem as an ill-posed integral equation, where the causal function is estimated by solving a conditional regression problem \cite{newey2003instrumental, hall2005nonparametric}. Kernel-based methods, particularly those employing reproducing kernel Hilbert spaces (RKHS), have been applied to nonparameteric estimation due to their flexibility and theoretical guarantees \cite{darolles2011, singh2019kernel}. However, these methods encounter fundamental limitations when $\beta\gg 1$ (high treatment dimension) even in the plain kernel ridge regression case with rotationally invariant kernels, leading to poorly conditioned kernel matrices and unstable estimations \cite{donhauser2022fast}.

The curse of dimensionality in kernel methods manifests in several ways: kernel matrices become increasingly ill-conditioned, with exponentially decaying eigenvalue spectrum (high condition number), which leads to worse performance \cite{steinwart2008support, donhauser2022fast} in the kernel ridge solution. This leads to ridge solutions that primarily retain high eigenvalue modes while discarding low eigenvalue directions, similar limitations have been noted in low-rank kernel approximations \cite{bach2013sharpanalysislowrankkernel}. In the context of NPIV estimation, the layer of conditional expectation adds a further layer of complication. The inherently ill-posedness of the integral equation, where small perturbations in the data can lead to large changes in the estimated causal function \cite{CARRASCO2007}. The ill-posedness of integral equations can be simply understood as attempting to recover information of the original function from the convolved function, which for a generic convolution kernel is not fully invertible. Existing regularization techniques such as Tikhonov regularization provide general suppression over all unstable eigendirections of the kernel function, can result in over-smoothed estimates that fail to capture complex nonlinear relations. This impedes the performance especially when those eigendirections align with the target causal function\footnote{Rather than kernel spectrum which we are discussing here, misaligned eigenspectrum in the conditional integral operator in IV was studied in \cite{meunier2025demystifyingspectralfeaturelearning}.}. 

Recent advances in understanding the fundamental limits of NPIV estimation have highlighted the polynomial approximation barrier where standard kernel methods require sample sizes that grow polynomially with dimension to maintain estimation accuracy \cite{donhauser2022fast}. This limit has motivated the development of alternative approaches, including deep learning methods \cite{xu2023learningdeepfeaturesinstrumental, kim2025optimalityadaptivitydeepneural} and other regularization schemes \cite{muandet2020dual}. 

In this paper, we introduce a novel perturbative renormalization approach to NPIV estimation that addresses the situations where the kernel matrix is low ranked or with high condition number (caused mainly by high dimensionality) while preserving the kernel framework's desirable properties. Our method draws inspiration from quantum physics, where perturbative expansions and renormalization techniques are used to handle divergent series and extract finite meaningful results from seemingly intractable calculations \cite{Peskin:1995ev, deFariadeMelo2010}. The key insight is that the kernel coefficients in NPIV estimation can be expanded as a power series in a coupling parameter, allowing for systematic correction of the base kernel ridge regression solution. Our approach consists of three main components: $(1)$ a perturbative expansion that introduces higher-order moment interactions to capture complex dependencies in high dimensional spaces, $(2)$ a renormalization procedure that tames the expansion by adaptively rescaling the higher order terms in the power series, and $(3)$ a resurgence technique that handles cases where the naive perturbative series becomes factorially divergent. From a spectral perspective, we demonstrate how these higher moment interactions induce perturbative solution effectively boost the contribution of eigenmodes with small eigenvalues and nonlinearly mix contributions from different eigenmodes\footnote{A different approach attempting to regularizing the kernel eigenmodes were discussed in the context of kernel MMD flow \cite{chen2024regularized, hagrass2024spectralregularizedkerneltwosample}.}. 

We demonstrate the effectiveness of our approach through extensive experiments on challenging high-dimensional NPIV problems. Our results show that Gaussian RBF kernels see minimal improvement due to their rotational invariance which almost zero them out in high dimensions, the family of fractional Brownian kernels \cite{Sejdinovic_2013} achieve substantial performance gains, with improvements of up to $99\%$ in mean square error compared to base kernel ridge regression. The method is particularly effective in regimes where the dimensionality grows rapidly with sample size, precisely where traditional NPIV methods struggle most.
Our approach opens new possibilities for tackling ill-conditioned problems in causal inference and suggests broader applications of perturbative methods in machine learning.

This paper is structured in the following way. In section \ref{sec:causal framework} and \ref{sec:NPIV_basic}, we introduce the basic approach to NPIV using kernel ridge regression. In section \ref{sec:perturbative_approach}, we describe the novel perturbative corrections we add to standard ridge regression and how we regularize the result. In section \ref{sec:implementation}, we discuss the implementations of the perturbative method and in section \ref{sec:experiments} we present and discuss the experiments. 
%%%%%%%%%%%%%%%%%%%%%%%%%
\subsection{Related works}

Nonparametric instrumental variable estimation has a rich history in econometrics and statistics. Early work by \cite{newey2003instrumental} and \cite{hall2005nonparametric} established the theoretical foundations, while subsequent research has focused on addressing the ill-posedness of the inverse problem inherent in NPIV. Kernel-based approaches to NPIV have been developed by \cite{singh2019kernel} and \cite{muandet2020dual}, leveraging reproducing kernel Hilbert spaces (RKHS) to represent the unknown structural function. These methods typically rely on Tikhonov regularization to ensure stability, but their performance degrades rapidly in high dimensions. The challenges of high-dimensional kernel ridge regression were systematically analyzed by \cite{donhauser2022fast}, who established minimax optimal rates and identified fundamental limits on the performance of polynomial approximation methods when the dimensionality parameter $\beta$ exceeds certain thresholds. Their analysis showed that standard methods face a "polynomial approximation barrier" as dimensionality increases. Various approaches have been proposed to address high-dimensional challenges in related contexts. \cite{belloni2015sparsemodelsmethodsoptimal} developed LASSO-based methods for high-dimensional instrumental variables, focusing on sparse linear models. Neural network approaches have been explored by \cite{hartford2017deep} and \cite{bennett2020deepgeneralizedmethodmoments, xu2023learningdeepfeaturesinstrumental, kim2025optimalityadaptivitydeepneural}, leveraging the representation power of deep learning for instrumental variable estimation. 

Our perturbative renormalization approach differs from existing methods by systematically incorporating higher-order corrections while ensuring stability through quantum field theory-motivated renormalization. This allows us to achieve performance improvements in dimensionality regimes where traditional methods struggle, without sacrificing interpretability or theoretical foundations.

% \textcolor{red}{cite the Volterra network papers, cite causal inference and kernel regression papers}

%%%%%%%%%%%%%%%%%%%%%%%%%
%%%%%%%%%%%%%%%%%%%%%%%%%
\section{Causal Framework}\label{sec:causal framework}

We consider a causal model with the following structure:

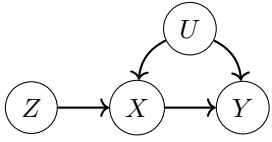
\begin{figure}[h]
\centering
\begin{tikzpicture}[scale=0.7]
    \node[circle, draw] (Z) at (0,0) {$Z$};
    \node[circle, draw] (X) at (2,0) {$X$};
    \node[circle, draw] (Y) at (4,0) {$Y$};
    \node[circle, draw] (A) at (3,1.5) {$U$};
    
    \draw[->, thick] (Z) -- (X);
    \draw[->, thick] (X) -- (Y);
    \draw[->, thick, bend left] (A) to (Y);
    \draw[->, thick, bend right] (A) to (X);
\end{tikzpicture}
\caption{Causal diagram: $Z$ is an instrument for $X$, while $U$ represents unobserved confounding between $X$ and $Y$.}
\end{figure}

In this framework, $Z$ is an instrumental variable that affects $X$ but not $Y$ directly. $X$ is the treatment/exposure variable that causally affects $Y$, which is the outcome of interest. Finally, $U$ represents unobserved confounding that affects both $X$ and $Y$.

The key assumptions are: $Z$ affects $Y$ only through $X$ (exclusion restriction); $Z$ has a non-zero effect on $X$; $Z$ is independent of the unobserved confounders $U$: $U \indep Z$ and $\E[U|Z] = 0$ and the confounding effect of $U$ on $Y$ is additive\footnote{If this is not the case, one could use proxy learning \cite{miao2018identifyingcausaleffectsproxy} which also leads to an integral equation with similar issues, the algorithm described in this paper would also improve the performance.}. The structural equation for the outcome $Y$ can be written as:
\begin{equation}
Y = g(X) + U\,,
\end{equation}
where $g(X)$ is the causal effect of interest, $U$ is the unobserved confounder. Our goal is to estimate the function $g(X)$ that represents the causal effect of $X$ on $Y$.

%%%%%%%%%%%%%%%%%%%%%%
%%%%%%%%%%%%%%%%%%%%%%
\section{Nonparametric Instrumental Variable Estimation (NPIV)}\label{sec:NPIV_basic}
Due to the confounding by $U$, standard regression of $Y$ on $X$ will not identify the causal effect $g(X)$. Instead, we use the instrument $Z$ to form moment conditions.
The key moment equation is:
\begin{equation}
\E[Y|Z] = \E[g(X)|Z]\,,
\end{equation}
after using the fact that $U \indep Z$ and $\E[U|Z] = 0$.

We represent the unknown function $g(X)$ using the reproducing kernel Hilbert space (RKHS) approach, with finite $n$ samples, the representer theory for kernel ridge regression dictates that:
\begin{equation}
g(x) = \sum_{i=1}^n \alpha_i K(x, x_i)\,,
\end{equation}
where $K(x, x')$ is a positive definite kernel, $x_i$ are stage 2 samples conditioned on $Z$ in the standard two-stage approach and $\alpha_i$ are coefficients to be determined from the standard representer theorem \cite{representer_theorem7}. 
%We show the standard representer theorem in appendix \ref{appendix:representer_theorem_generalized}.

To estimate $g(X)$, we formulate the following objective function:
\begin{equation}
S_0 = \E\left[\left(\E[Y|Z] - \int g(x)f(x|Z)dx\right)^2\right]\,,
\end{equation}
where $f(x|Z)$ is the conditional density of $X$ given $Z$, the outer expectation is taken with respect to $Z$.

Substituting the kernel representation, we get:
\begin{equation}
S_0 = \E\left[\left(\E[Y|Z] - \int \sum_{i=1}^n \alpha_i K(x, x_i)f(x|Z)dx\right)^2\right]\,.
\end{equation}
To ensure stability of the solution, we add a quadratic regularization term:
\begin{equation}
S_{\text{ridge}} = S_0 + \lambda\|g(x)\|^2_{\mathcal{H}}\,,
\end{equation}
where $\|g(x)\|^2_{\mathcal{H}}$ denotes the RKHS norm of $g(x)$. To derive the estimator, we first define:
\begin{align}
&h_i:= \E\left[\int K(x, x_i)f(x|Z)dx  \cdot \E[Y|Z]\right]\,;\\
&\widetilde{K}_{ij} := \E\left[\iint K(x, x_i)K(x', x_j)f(x|Z)f(x'|Z)dxdx'\right]\,; \nonumber
\end{align}
where the outer expectation is taken with respect to $Z$. Now we can rewrite the objective function as (omitting the $\E[Y|Z]^2$ term which is irrelevant from the RKHS coefficient $\alpha_i$):
\begin{equation}
S_{\text{ridge}} = \sum_{i,j} \widetilde{K}_{ij}\alpha_i\alpha_j - 2\sum_i \alpha_i h_i + \lambda\sum_{i,j} \alpha_i\alpha_j K_{ij}\,,
\end{equation}
where $K_{ij}$ is the standard Gram matrix elements, quite distinct from $\widetilde{K}_{ij}$ which is the density smoothed version. Taking the functional derivative with respect to $\alpha_i$ and setting it to zero\footnote{For a detailed discussion of functional derivatives and how they act, we refer the readers to the appendix \ref{appendix:representer_theorem_generalized}.}:
\begin{equation}\label{eq:order_0}
\frac{\delta S_{\text{ridge}}}{\delta \alpha_i} =0 \Rightarrow \sum_j \widetilde{K}_{ij}\alpha_j + \lambda K_{ij}\alpha_j = h_i\,.
\end{equation}
In matrix form $(\widetilde{K} + \lambda K)\alpha = h$. The solution is the standard kernel ridge, denoted as $\alpha^{(0)}$, the zeroth-order solution in our perturbative approach.

%%%%%%%%%%%%%%%%%%%%%%%%
%%%%%%%%%%%%%%%%%%%%%%%%
\section{Perturbative Approach}\label{sec:perturbative_approach}
When condition number $\kappa$ is high or the effective rank is low, a well-understood problem with kernel methods is that eigendirections with small eigenvalues become under-represented in the solution space, which suggests a fast decay in the eigenvalues. This would not be an issue if the target function $h$ happen to have features only in the eigendirections with large eigenvalues. However, in the NPIV setting, in the presence of a convolution integral equation, for example in deep learning IV \cite{wiltzer2024foundationsmultivariatedistributionalreinforcement, xu2023learningdeepfeaturesinstrumental}, the fast spectrum decay admits less and less effective eigenmodes. High condition number can also occur when instrumental variables are weak where changes in $Z$ causes few to no changes in $g(x)$ \cite{weakinstrumental, Stockweakinstrumental}.

In order to tackle this, we need to induce mixing between different eigendirections by remixing their weights in the regression solution\footnote{We do the spectral analysis in appendix \ref{appendix:spectral_explanation_of_triple_moment_int}.} To this end, we introduce a perturbative solution to kernel regression by incorporating additional higher-order interactions\footnote{This is rather standard approach from physics, for a reference chapter \cite{deFariadeMelo2010}, and \cite{Peskin:1995ev}, where when perturbatively studying the scattering problem in quantum electrodynamics, a electron-positron-photon-photon quartic interaction is added. Also used in standard PDE solving and quantum mechanics \cite{Dyson:1952tj, Volin:2009wr, Magnen:1977ha, Lipatov:1976ny}, usually quoted as the WKB approximation method.} The simplest term one can add in the objective function is the triple moment, we define:
\begin{equation}
K_{ijk} = \E\left[\int K(x, x_i)K(x, x_j)K(x, x_k)f(x|Z)dx\right]
\end{equation}
where $x$ labels an independent draw from the probability density $f(x|Z)$. Intuitively, this includes the potential influence of kernel functions from three different locations in $\mathbb{R}^n$ on the regressed function when mapped to the Hilbert space. We start with the objective function with a cubic term,
% \begin{strip}
%     \begin{equation}
%         S = \frac{1}{2}\left(\E[Y|Z] - \int g(x)f(x|Z)dx\right)^2 + \frac{\lambda}{2}\sum_{i=1}^n \alpha_i^2 K_{ii} 
% + \frac{\gamma}{6}\sum_{i,j,k} \alpha_i\alpha_j\alpha_k\int K(x,x_i)K(x,x_j)K(x,x_k)f(x|Z)dx
%     \end{equation}
% \end{strip}
\begin{align}\label{eq:int_objective}
S_{\text{pert}} = S_{\text{ridge}} + 
\frac{2\gamma}{3}&\sum_{i,j,k} \alpha_i\alpha_j\alpha_k K_{ijk}\,.
\end{align}
Taking the derivative with respect to $\alpha_i$, we obtain the minimization criteria for the objective function:
\begin{align}\label{eq:1}
&\frac{\delta S_{\text{pert}}}{\delta \alpha_i} =2\,\E\left[\lambda K_{ij}\alpha_j -\left(\E[Y|Z] - \int g(x)f(x|Z)dx\right)\times\right. \nonumber\\&\qquad\quad\int K(x',x_i)f(x'|Z)dx' +  
\\ &\left.\gamma\sum_{j,k} \alpha_j\alpha_k\int K(x,x_i)K(x,x_j)K(x,x_k)f(x|Z)dx\right] \nonumber\,.
\end{align}
We still expand the target function $g(x)$ in the kernel Hilbert space, which traditionally in the ridge regression case relies on the optimal representer theorem. We prove the generalized representer theorem in appendix \ref{appendix:representer_theorem_generalized} in the presence of the triple moment term. Using the representer theorem like before, we simply set the variation \eqref{eq:1} to zero and collecting terms, we arrive at the optimality condition on the kernel coefficients $\alpha_i$ that yields our perturbative expansion.
\begin{equation}
    \sum_{j} (\widetilde{K}_{ij}+\lambda K_{ij})\alpha_j + 3\gamma\sum_{j,k}K_{ijk}\alpha_j\alpha_k = h_i\,,
\end{equation}
where we have compactified the integrals, $\gamma$ is a coupling parameter that controls the strength of the three-point interaction (three point moment), we are essentially perturbing the regression problem around the zeroth order "free" theory by adding this term. 

%%%%%%%%%%%%%%%%%%%%%%%%%%%
\subsection{Perturbative Expansion}
Due to the addition of triple moment term with scalar $\gamma$, we expect the regressed result $\alpha_i$ to depend on $\gamma$, which should reduce to the usual kernel ridge regression solution \eqref{eq:order_0}. A naive first ansatz one could use is simply the power series. This is only a first guess so far, we are certainly not guaranteed to have an everywhere convergent series. In fact the experience from physics is that they are almost never convergent, accordingly there are procedures to tame the divergent ones we shall introduce in the next two subsections. Expand the coefficients as a power series in $\gamma$
\begin{equation}
\alpha = \alpha^{(0)} + \gamma\alpha^{(1)} + \gamma^2\alpha^{(2)} + \ldots\,.
\end{equation}
Substituting this into the optimality condition and collecting terms of the same power in $\gamma$, we get:
\begin{multline}
    \sum_j (\widetilde{K}_{ij}+\lambda K_{ij})\left(\sum_{l=0}^\infty\gamma^l\alpha^{(l)}_j\right) =\\
    h_i - \gamma\sum_{j,k}K_{ijk}\left(\sum_{n=0}^\infty \gamma^n\alpha^{(n)}_j\right)\left(\sum_{m=0}^\infty\gamma^m\alpha^{(m)}_k\right)\,.
\end{multline}
The benefit of expanding $\alpha_i(\gamma)$ as a power series means that we can now use the simple trick of coefficient matching:
\begin{equation}
    \left.\frac{\partial^n}{\partial\gamma^n}\left(\frac{\delta S_{\text{pert}}}{\delta \alpha_i}(\gamma)\right)\right\vert_{\gamma=0} = 0 \textbf{ for  $n\in\mathbb{N}$}\,.
\end{equation}
It is easy to see that we can solve for the $\alpha^{(n)}$s order by order iteratively. Looking at the coefficient of $\gamma^0$ term from both sides of the equation:
\begin{equation}
(\widetilde{K} + \lambda K)\alpha^{(0)} = h\,.
\end{equation}
This is the usual ridge regression term we talked about before. Similarly for the coefficient of $\gamma^1$, we have:
\begin{equation}
(\widetilde{K} + \lambda K)\alpha^{(1)} = -3\sum_{j,k} K_{ijk}\alpha^{(0)}_j\alpha^{(0)}_k\,,
\end{equation}
although notationally we omitted the $x_i$ dependence, hopefully it is self-explanatory\footnote{In full component form $(\widetilde{K}_{im} + \lambda K_{im})\alpha_m^{(1)} = -3\sum_{j,k} K_{ijk}\alpha^{(0)}_j\alpha^{(0)}_k$.} 
% and subsequently for $\gamma^2$:
% \begin{equation}
% (\widetilde{K} + \lambda K)\alpha^{(2)} = -3\sum_{j,k} K_{ijk}(\alpha^{(1)}_j\alpha^{(0)}_k + \alpha^{(0)}_j\alpha^{(1)}_k)\,.
% \end{equation}
and subsequently the coefficient of generic $\gamma^n$ term:
\begin{equation}
    \begin{aligned}
    &(\widetilde{K} + \lambda K)\alpha^{(n)}=-3\times\\
    &\sum_{j,k} K_{ijk}\left(\alpha_j^{0}\alpha_k^{n-1}+ \alpha_j^{1}\alpha_k^{n-2}+\dots+ \alpha_j^{n-1}\alpha_k^{0}\right)\,.
\end{aligned}
\end{equation}
We see that this is an iterative algorithm, where the computation of each new perturbative order is a solving simple ridge regression problem on all the previous perturbative order coefficients\footnote{This is equivalent to a technique familiar in perturbative quantum field theory called the Feynman diagram, where the sum is computed by drawing all possible diagrams whose external legs are the $\alpha_i^{(n)}$s.}. 

This creates non-trivial nonlinear mixing among different frequency modes and rescue the representation power of the kernel in high dimensions when condition number is high by boosting the presence/occurrence of eigenmodes with small eigenvalues. Explicitly, for cleaner presentation, the first two orders of contributions to the function $g(x)$ can be written as a sum over different eigen modes of the kernel matrix $\phi_m$ and their corresponding eigenvalue $\lambda_m$ (in the case when $\widetilde{K}$ is diagonalizable in the kernel eigenbasis)\footnote{The case where $\widetilde{K}$ is not diagonalizable in the kernel eigenbasis is slightly more involving, we present the easier version here for presentation purpose.}:
\begin{align}
    &g^{(0)}(x) =\sum_i\sum_{m=1}^N\frac{\lambda_m}{\eta_m+\lambda}\phi^T_m h \,\phi_m
    \\
    &g^{(1)}(x) =\sum_i\sum_{m,n,p,q,r}\frac{\lambda^2_m\lambda_n\lambda_p}{\eta_m+\lambda}\frac{\phi^T_q h}{\eta_q+\lambda}\frac{\phi^T_r h}{\eta_r+\lambda} \times \tilde\Phi_{mnpqr} \nonumber\,,
\end{align}
where $\tilde\Phi_{mnpqr}$ are various combinations of the eigenmode functions. $\eta_m$ are square root of the eigenvalues of the kernel smoothed $\widetilde{K}$ in the kernel eigenbasis, which is slightly smaller than $\lambda_m$ due to the variance contraction of the conditional integral operator. To simplify discussions, we have written the expression when $\eta_m\sim\lambda_m$, which essentially trivializes the conditional integral and brings us back to standard kernel ridge regression. We expand further into the derivation of this in appendix \ref{appendix:spectral_explanation_of_triple_moment_int}. We note that generically, $\widetilde{K}$ would not be diagonalizable in the eigenbasis of the kernel matrix, which in some sense, induces spectral mixing between different eigenmodes of the kernel. The spectral result we provide here is illustrating that on top of the linear mixing induced by the conditional expectation, the triple moment interaction term introduces additional non-linear spectral mixing. 

In this simple setting, the order $0$ contribution is the familiar spectral representation of standard kernel ridge regression. We see that due to the presence of the regularization scale $\frac{\lambda_m}{\lambda_m+\lambda}$, the contributions to $g(x)$ only come from larger $\lambda_m$s while eigenmodes with $\lambda_m\ll \lambda$ are suppressed. This is quantified by the condition number $\kappa = \lambda_{\text{max}}/\lambda_{\text{min}}$. High $\kappa$ signals the representation power of the kernel becomes concentrated to the first few large eigenvalues, effectively making other tail eigenmodes irrelevant (due to the ridge regularization serving as a cut-off threshold). On the other hand, the first order contribution $g^{(1)}(x)$ induces nonlinear spectral mixing, which even when $\kappa\gg 1$, reinstates the contribution from tail eigenmodes by pairing them with the larger eigenmodes\footnote{We note that this is certainly not the only possible term on can add in the objective function, one can reverse engineer any term that does not violate the representer theorem starting with the desired spectral properties.}. This remains true in the generic NPIV setting where $\widetilde{K}$ is non-diagonalizable in the eigenbasis of the kernel. 

%%%%%%%%%%%%%%%%%%%%%%%%%
\subsection{Renormalization}
As we mentioned earlier, the naive power series ansatz we assumed could be divergent, giving meaningless answers in the perturbative result $\alpha_i$. This can become dramatic especially when the inverse problem is more and more ill-defined, the norms of the coefficient vectors $\alpha^{(k)}$ can grow rapidly, leading to numerical instability. Physics literature has explored multiple approaches to address this divergence. One particular way, referred to as Wilsonian renormalization \cite{Wilson:1973jj}, imposes a cut-off/regularization parameter on physical systems (the energy spectrum for example) and rely on the belief that physical systems should not depend on such artificially imposed cut-off. This reveals the physically meaningful part of the divergent series which is stable under this process. 

Inspired by this, although the regularization parameter in the RKHS ($\lambda$ regularization strength in kernel ridge regression) is not physical\footnote{For readers familiar with the physics concept of renormalization might wonder if this is connected with the renormalization group equation, we discuss this further in appendix \ref{appendix:irrelevance_of_physical_renormalization}, asserting the differences with the usual renormalization and explaining why the physical assumption does not apply here.}, we introduce renormalization parameters as a simple rescaling $\alpha = \alpha^{(0)} + \tilde\gamma_1\alpha^{(1)} + \tilde\gamma_2\alpha^{(2)} + \ldots$ where $\tilde\gamma_k$ are the renormalized coefficients. Given by a simple rescaling:
\begin{equation}
    \tilde\gamma_k = \gamma^k \min\left(\frac{\|\alpha^{(0)}\|}{\|\alpha^{(k)}\|}, 1\right) \,,
\end{equation}
as we shall further demonstrate in appendix \ref{subsec:need_for_renorm} using the spectral property of the kernel operators, the inevitable divergence of the naive power series, this simple rescaling tames the divergence. 

%%%%%%%%%%%%%%%%%%%%%
\subsection{Resurgence}\label{subsec:resurgence}
Yet another tool to tame potential divergence in the perturbative ansatz and extract useful information is resurgence \cite{ecalle1981fonctions,Berry1991}. In certain cases, the series could be an asymptotic series, which is factorially divergent in its coefficients $\alpha_i^{(n)}$. 

This happens simply because we have chosen the point $\gamma=0$ to perturbatively expand (which is assumed to be a saddle in the objective function landscape $S$), this may very well be incorrect due to the presence of numerous other saddle points. Luckily, resurgence comes to rescue, we give a much more detailed introduction on this subject practically in appendix \ref{appendix:resurgence}. The belief of the resurgence trick is that divergent series are divergent for a reason, they diverge due to the existence of other nearby saddle points in the function landscape. Naive perturbative expansion "falls off the cliff" along the directions where the current saddle is a maxima. Resurgence essentially extracts the nearby saddle point information from the divergent higher order terms which have factorially divergent coefficients using complex analysis, then rewriting the series to account for the existence of the nearby saddle points, we recover a trans-series, which has finite radius of convergence. Resurgence/Borel resummation finds ways to take factorially divergent series and resum them to an integral form. Although this integral is not well-defined around singularities in its domain, one can resolve this by adding additional series around those points suppressed by the ambiguity value. This returns a convergent series with no ambiguity. Curious readers are referred to appendix \ref{appendix:resurgence} for more detail.

For now, on a practical level, we shall summerize things as an algorithm if we run into factorially divergent coefficient functions. Check if $\alpha_{i}^{(n)}\sim n!$, if false, then series not asymptotic, just stick to renormalized series. If true, then series asymptotic, use Borel resummation $\hat{g}(tX) = \sum_i\sum_{n=0}^\infty \frac{\alpha_{i}^{(n)}}{n!}\gamma^n K(tX, X_i)$. Then one examines singularity on complex $t$-plane. This gives a collection $\{t_k\}_{k=1}^m$ singularities. Computing the residue around those singularities returns a trans-series with no ambiguity:
\begin{equation}
    \Phi(X) = \underbrace{\sum_{n=0}^\infty\alpha^{(n)}\gamma_n}_{\text{original series}}  + \underbrace{\sum_{k=1}^m C_k\,e^{-k \,t_k(X)/X}g_k(X)}_{\text{non-perturbative near saddle}}\,,
\end{equation}
with $C_k = S_{C_+} - S_{C_-}= \text{Res}_{t=t_k}\hat{g}(tX)$ is the ambiguity of choosing complexified integration contour above or below the singularity named Stokes constant. 

% This procedure starts with a perturbative series which is Borel resummable across the entirety of the complex plane and give us the original integral back, but there is an ambiguity near the singularity, when the integration contour goes across the singularity, we see a jump in the result. This is resolved by adding a term that jumps in the opposite way (the non-perturbative term). Then we have a non-ambiguous definition of a series\footnote{This is done in the formal Borel-Escalle theory by computing the Lefschetz thimbles, we shall not delve into that.}. 

The one line take away message of this part is that when the perturbative series $\alpha_i(\gamma)$ is factorially divergent, one could use the method described above to obtain a series with finite (nonzero) radius of convergence and extract useful information from the originally wildly divergent series.

%%%%%%%%%%%%%%%%%%%%%%%%
%%%%%%%%%%%%%%%%%%%%%%%%
\section{Implementation}\label{sec:implementation}
To summerize the perturbative procedures we discussed so far, we practically have two stages. First compute the kernel matrices needed, then iteratively compute the raw RKHS coefficients $\alpha_i(\gamma)$ as a power series in $\gamma$ up to order $N$. Then we handle the case if the coefficients $\alpha_i^{(n)}$ are divergent, we first use rescaling. If they are factorially divergent, we further use resurgence.  
%%%%%%%%%%%%%%%%%%%%%%%%%%%%%%%%%%%%%%%%%
\subsection{Computing the Kernel Matrices}
For implementation, one can use the standard Nadaraya-Watson estimator to estimate the two conditional densities $\E[Y|Z]$ and $f(X|Z)$ after picking the kernel. We do this for the triple moment kernel:
\begin{equation}
K_{ijk}\approx \frac{1}{n}\sum_{r=1}^n K(x_r, x_i)K(x_r, x_j)K(x_r, x_k)\cdot w_r\,,
\end{equation}
where $w_r$ is the weight that account for the conditional density $f(x|Z)$, which is estimated using standard Nadaraya-Watson estimator. And similarly for 
\begin{equation}
    \E[Y|Z=z] \approx \frac{\sum_{i=1}^n l(z, z_i)y_i}{\sum_{i=1}^n l(z, z_i)}\,,
\end{equation}
where $l(z, z')$ is a kernel function for the instrument space.
More traditionally in the usual kernel ridge regression, instead of directly estimating the density $b_i(Z) = \E[K(x,x_i)|Z]$, people use the two stage method where the first stage is to compute the conditional mean embedding directly in RKHS via the covariance operator. For each $z\in Z$, the conditional feature vector is $b(z) = K_X(L+n\rho I)^{-1}l_z$, where $(K_X)_{ij} = k(x_i, x_j)$ and $L_{kl} = l(z_k,z_l)$ are Gram matrices on X and Z, $l_z=[l(z_1,z), \dots,l(z_n, z)]^T$ denotes the row vector, $n$ is the dimension of the instrumental while $\rho$ is some regularization parameter. Then computing $\E[Y|Z] = \frac{1}{n}\,l_z^T(L+n\rho I)^{-1} y$. The desired quantities are computed using $b_i(Z)$ and $\E[Y|Z]$:
\begin{equation*}
    h_i = \E[b_i \E[Y|Z]]\,, K_{ij} = \E[b_i b_j] = \frac{1}{n}\sum_{l=1}^n b_i(z_l) b_j(z_l)\,.
\end{equation*}

%%%%%%%%%%%%%%%%%%%%%%%%
\subsection{Algorithm}
For prediction on new points $x_{new}$, we compute:
\begin{multline}
g(x_{new}) = \sum_i \alpha_i^{(0)} K(x_{new}, x_i) + \sum_i \gamma_1 \alpha_i^{(1)} K(x_{new}, x_i) \\
+ \sum_i \gamma_2 \alpha_i^{(2)} K(x_{new}, x_i) + \ldots\,,
\end{multline}
where $\alpha_i^{(n)}$ are the order by order RKHS coefficients induced by the triple moment and $\gamma_n$ are the renormalized coupling powers. The sum is truncated at a certain order $N$, where the marginal gain of computing a new order no longer beats the increase in computation time. 

The general procedure can be summerized in the following fashion. We first compute the needed kernel matrices $K_{ij}$, triple moment $K_{ijk}$ and conditional mean $\E[Y|Z]$ through conditional mean embedding and Nadaraya-Watson estimator. Then we compute the perturbative series in kernel coefficients $\alpha_i = \sum_{n=0}^N \gamma_n\,\alpha_i^{(n)}$ order by order solving ridge regression problems up to a cut-off order $N$. If we detect divergence in the series, rescale (renormalize) the coupling strength adaptively, if the series is factorially divergent, use resurgence algorithm. A detailed description for the first three steps is as in \ref{algo:renormalization}, which has an time complexity of $\mathcal{O}(n^3\times M)$ with $n$ data points computed to $M$ orders, prior to this, computing $K_{ijk
}$ requires a further $\mathcal{O}(n^4)$ complexity computation. We deferred the detail of last resurgence step to the appendix \ref{algo:resurgence}.

%%%%%%%%%%%%%%%%%%%%%%%%%%%%%%
%%%%%%%%%%%%%%%%%%%%%%%%%%%%%%
\section{Experimental Results}\label{sec:experiments}

We conducted extensive experiments to evaluate the performance of our perturbative renormalization approach across different dimensionality regimes. The experiments were done with $n=80$ data points generated around the ground truth, which are the datasets described in appendix \ref{sec:dataset}, which create a series of challenging high-dimensional NPIV problem with controlled properties including standard IV benchmarks. Here we briefly summerize the overall performance of the algorithms across different kernel types: Gaussian RBF kernel and fractional Brownian kernel. As a comparison, we also add the fitting RMSE of Deep IV method \cite{hartford2017deep} on the same dataset after $500$ epochs of training each regression stage with approximately 2k parameter neural nets. We leave the complete tables of data across dimensions and renormalization strength in appendix \ref{appendix:detailed_results}. 
% \begin{table*}[htbp]
% \centering
% \caption{Performance Summary of NPIV Methods by Kernel Type}
% \label{tab:npiv_kernel_summary}
% \begin{tabular}{l|ccc|cc}
% \toprule
% \multirow{2}{*}{\textbf{Kernel}} & \multicolumn{3}{c|}{\textbf{Mean RMSE}} & \multicolumn{2}{c|}{\textbf{Improvement}}  \\
% & \textbf{Order 0} & \textbf{Renorm}& \textbf{Trans} & \textbf{Mean $\Delta\%$} & \textbf{Best $\Delta\%$} \\
% \midrule
% RBF  & 0. & 0.621 & 0.620 & $-0.1$ & $0.0$ \\
% Fractional Brownian  & 0.711 & 0.585 & 0.484 & $+21.7$ & $+85.0$   \\
% \bottomrule
% \end{tabular}
% \footnotesize
% \begin{tablenotes}
% \item Order 0 = baseline kernel ridge regression; Renorm = renormalized series; Trans = trans-series with resurgence
% \item $\Delta\%$ = percentage improvement over Order 0 baseline
% \end{tablenotes}
% \end{table*}

\begin{table*}[htbp]
\centering
\caption{NPIV best Performance Scaling with Dimension}
\label{tab:beta_sweep_all_kernels}
\begin{tabular}{cc|ccc|ccc|c}
\toprule
\multirow{2}{*}{$\beta$} & \multirow{2}{*}{$d$} & \multicolumn{3}{c|}{\textbf{RBF}} & \multicolumn{3}{c|}{\textbf{fractional Brownian}} & \textbf{Deep IV} \\ & & Order 0 & Best Pert & Best $\Delta$\% & Order 0 & Best Pert & Best $\Delta$\%  & RMSE\\
\midrule
0.3 & 3 & 0.317 & 0.318 & $-0.2$ & 0.235 & 0.238 & $-1.2$ & \textbf{0.121}\\
0.5 & 7  & 0.536 & 0.536 & $+0.0$ & \textbf{0.173} &  0.182 & $-5.2$ & 0.187\\
0.7 & 17 & 0.289 & 0.289 & $+0.0$  & 0.382 & \textbf{0.124} & $+67.5$ & 0.225\\
1.0 & 60 & 0.360 & 0.360 & $+0.0$ & 1.654 & 0.108 & $+93.5$ & \textbf{0.089}\\
1.3 & 204 & 0.511 & 0.511 & $+0.0$ & 3.766 & \textbf{0.028} & $+99.3$ & 0.078\\
1.5 & 464 & 0.465 & 0.465 & $+0.0$ & 11.52 & 0.136 & $+98.8$ & \textbf{0.096}\\
\bottomrule
\end{tabular}
\footnotesize
\begin{tablenotes}
\centering
\item $\beta$ controls dimension $d = n^\beta$ where $n = 80$; values taken over best among $\gamma \in \{0.4, 0.6, 0.8, 0.99\}$
\item "Order 0" is kernel ridge regression RMSE\footnote{The results presented are all cross validated for different choices of the L2 regularization parameter $\lambda$.}; "Best Pert" refers to the best RMSE obtained using perturbative method; "Best $\Delta\%$" represents the percentage gain in RMSE.
\end{tablenotes}
\end{table*}

\begin{figure}[ht]
    \centering
    \includegraphics[width=0.75\linewidth]{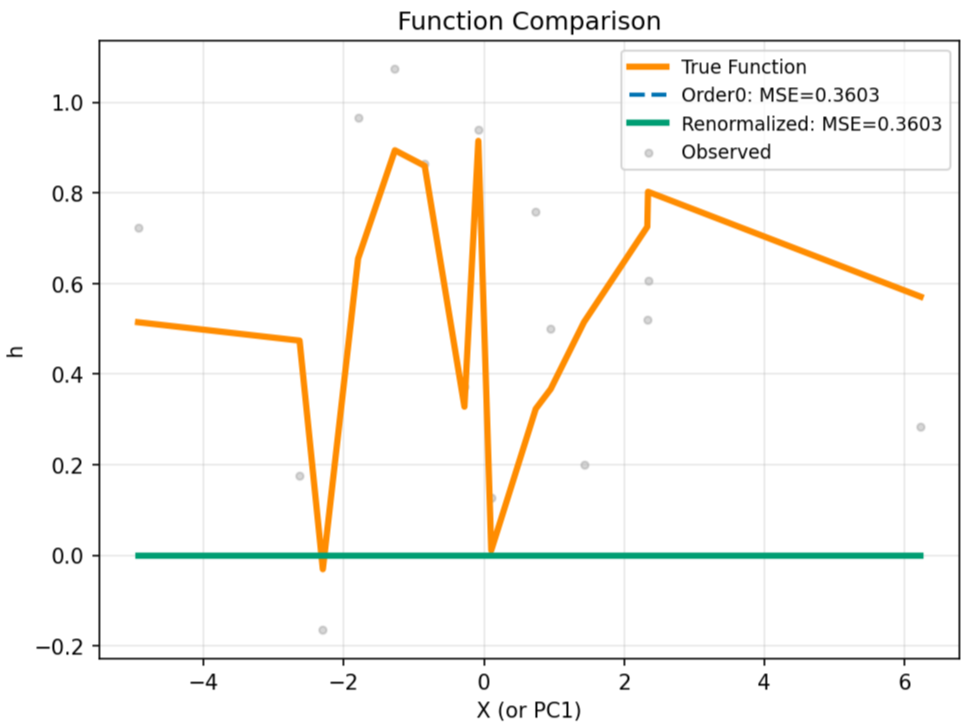}
    \caption{Example RBF kernel fitting}
    \label{fig:RBF_kernel_example}
\end{figure}

\begin{figure}[ht]
    \centering
    \includegraphics[width=0.75\linewidth]{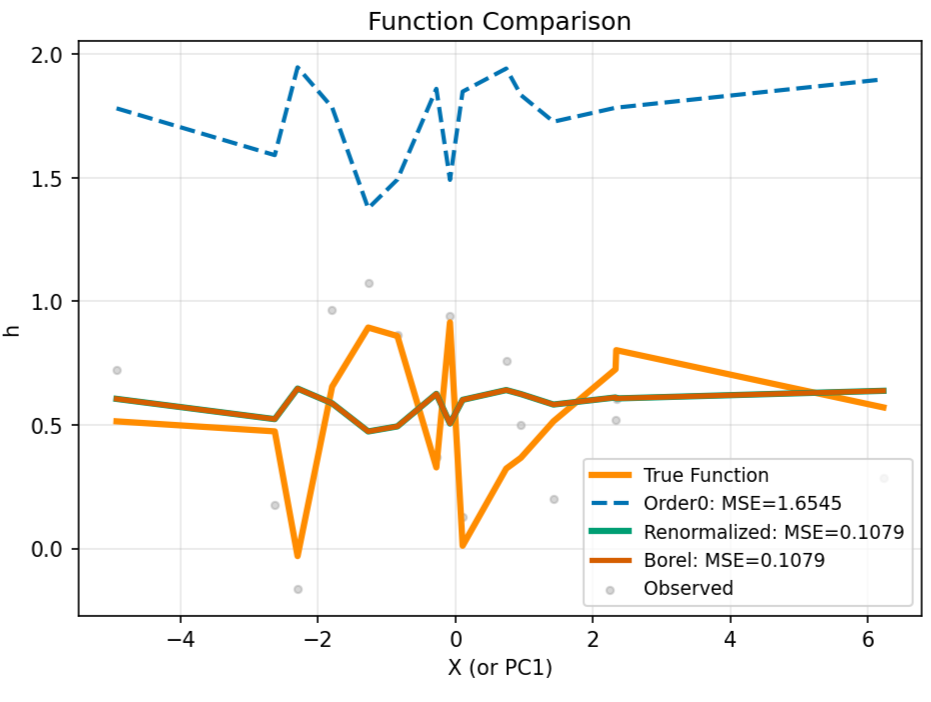}
     \caption{Example fractional Brownian kernel fitting}
    \label{fig:FB_kernel_example}
\end{figure}

In appendix \ref{subappendix:further_datasets}, we also include further standard NPIV datasets: Newey-Powell \cite{newey2003instrumental}, weak/strong instrumental datasets, heteroscedastic dataset, nonlinear instrumental dataset and sparse signal dataset. The performance result is summerized here, note that we only focus on the improvement against regularized kernel IV (order $0$) using fractional Brownian kernel, Deep IV method is there for completeness:

\begin{table*}[htbp]
\centering
\caption{RMSE Comparison on Alternative NPIV Datasets}
\label{tab:alternative_datasets_rmse}
\begin{tabular}{l|ccc}
\toprule
\textbf{Dataset} & \textbf{DeepIV} & \textbf{Order 0} & \textbf{Best Pert.}  \\
\midrule
Newey-Powell & \textbf{1.130} & 14.858 & 8.718  \\
Weak IV ($\rho=0.1$) & 0.347 & 2.779 & \textbf{0.280}  \\
Weak IV ($\rho=0.3$) & \textbf{0.289} & 2.165 & 0.339  \\
Strong IV ($\rho=0.7$) & \textbf{0.255} & 2.175 & 0.329  \\
Heteroscedastic & \textbf{0.678} & 7.969 & 0.724\\
Nonlinear IV & \textbf{0.388} & 14.986 & 0.409 \\
Sparse ($d=50$) & 0.586 & 13.833 & \textbf{0.560}  \\

\bottomrule
\end{tabular}
\footnotesize
\begin{tablenotes}
\item RMSE = $\sqrt{\text{MSE}}$; ``Order 0'' is standard kernel ridge IV with fractional Brownian kernel; ``Best Pert.'' is the best result across $\gamma \in \{0.4, 0.6, 0.8, 0.99\}$ with fractional Brownian kernel. Bold indicates best method per dataset.
\end{tablenotes}
\end{table*}

In addition to these, we further include empirical analysis on the performance of the algorithm on IV datasets with weak instrumentals in appendix \ref{sec:weak_instruments}, a sensitivity analysis on the parameters used in the algorithm (regularization parameter $\gamma$, ridge regularization parameter $\lambda$ and maximum order of perturbation $N_{\text{max}}$) in appendix \ref{sec:sensitivity} and experiments with larger sample size in appendix \ref{subappendix:largersampleexperiments}.

% Summary text for paper
\textbf{Key Findings:} The RBF kernel shows minimal improvement across different kernel widths. If one plots the fitted curve on cross sections of the hypersurface as in figure \ref{fig:RBF_kernel_example}, where the cross section is taken with respect to dimension $1$, we see that RBF kernel almost zeros out due to its rotational invariance forces concentration in high dimensions, this almost trivial fitting gives a convenient baseline MSE for each $\beta$ case. We see in the table \ref{tab:beta_sweep_all_kernels} that perturbative method we have here also matches that baseline, it simply overlaps with the trivial order $0$ kernel ridge regression fit in \ref{fig:RBF_kernel_example} since the added interaction is also constrained by rotational invariance. In appendix \ref{appendix:spectral_explanation_of_triple_moment_int} we further discuss why the extra term we add works from a spectral perspective, while explaining why it respects rotational invariance hence contributing minimally to the fitted result. We further propose another different kernel that breaks this invariance hence helping the performance of RBF kernels in lower dimensions in appendix \ref{appendix:new_int_term_breaking_rotationally_invariance}. Although in high dimensions, the curse of dimensionality forces everything to fit trivially, making RBF kernels beyond rescue. 

The fractional Brownian kernel, on the other hand, shows superior performance over the trivial RBF kernel fit at order $0$ when dimension is low $\beta < 0.7$ as in the first two rows of the table \ref{tab:beta_sweep_all_kernels}. At $\beta \geq 0.7$, we see that it actually has worse RMSE than the trivial RBF fit, as plotted in \ref{fig:FB_kernel_example}, this is due to the non-trivial wiggle attempting to non-trivially fit the curve. The perturbative method effectively brings the fitted curve down to the true function, significantly improving the fitted curve. The perturbatively revised fitting consistently outperforms the trivial RBF kernel baseline as well. It is not obvious from the diagram whether resurgence was also used (referred to as "Borel" in the legend), since it completely overlaps with the pure perturbative and renormalization ("Renormalized") line, the mechanism only comes into effect when the series is factorially divergent, which is not the case in majority of the cases.  

We note that the two algorithms we wrote \ref{algo:renormalization}, \ref{algo:resurgence} need to be comprehensively used to maximize their effectiveness. In that \ref{algo:resurgence} is only useful when the perturbative series in kernel coefficients $\alpha$ is asymptotic or factorially divergent. The resummation hurts the performance when this is not the case. It is not hard to see that this depends on the perturbative strength parameter $\gamma$, it is easier to obtain an asymptotic series when $\gamma$ is closer to $1$. $\gamma\to 1$ suggests more significant contributions from higher order terms in \ref{algo:renormalization}, which might hurt the its performance. Hence there is a subtle empirical balance in the choice of $\gamma$ to balance them. 

Overall, solving the integral equation amplifies the ill-conditioned-ness of the kernel matrices in high dimensions (high $\beta$/dimension-sample ratio). Perturbative approach \ref{algo:renormalization} rescues this by signifying the contribution from eigendirections with small eigenvalues, hence enhancing the expressivity of the kernel. Resurgence algorithm \ref{algo:resurgence} further corrects the behavior of the perturbative series when it is factorially divergent. This is a novel method of extracting the most out of kernel ridge regression. We note that the form of the additional term we added in the objective \eqref{eq:int_objective} can certainly be replaced with other ones, spectrally one could tailor the additional term needed in order for the induced perturbative solution to enhance a certain aspect of typical kernel ridge regression. We demonstrate this in appendix \ref{appendix:new_int_term_breaking_rotationally_invariance}, where a new term in the objective is introduced that explicitly breaks rotational invariance, tailored to tackle the extreme vanishing of rotationally invariant kernels in high dimensions. We hope this paper--although as a crude first try--demonstrates the applicability of perturbative methods (and the methods taming them) in statistical/econometrics world\footnote{Rigorous statistical proofs are not present, such as those finite sample bounds for standard kernel ridge regression \cite{hang2018optimallearninganisotropicgaussian,fischer2020sobolevnormlearningrates,Caponnetto2007OptimalRF}, this is due to the proposed algorithm does not satisfy the standard assumptions under which most of the estimator efficiency proofs are constructed, we leave this for future work.}.

\bibliography{example_paper}
\bibliographystyle{icml2026}

%%%%%%%%%%%%%%%%%%%%%%%%%%%%%%%%%%%%%%%%%%%%%%%%%%%%%%%%%%%%%%%%%%%%%%%%%%%%%%%
%%%%%%%%%%%%%%%%%%%%%%%%%%%%%%%%%%%%%%%%%%%%%%%%%%%%%%%%%%%%%%%%%%%%%%%%%%%%%%%
% APPENDIX
%%%%%%%%%%%%%%%%%%%%%%%%%%%%%%%%%%%%%%%%%%%%%%%%%%%%%%%%%%%%%%%%%%%%%%%%%%%%%%%
%%%%%%%%%%%%%%%%%%%%%%%%%%%%%%%%%%%%%%%%%%%%%%%%%%%%%%%%%%%%%%%%%%%%%%%%%%%%%%%
\newpage
\appendix
\onecolumn
\section{Spectral explanation of the triple moment term}\label{appendix:spectral_explanation_of_triple_moment_int}
\subsection{Mercer Decomposition and Covariance Operator}

Every positive definite kernel admits the representation:
\begin{equation}
k(x,y) = \langle \varphi(x), \varphi(y) \rangle = \sum_{m=0}^{\infty} \lambda_m g_m(x) g_m(y)\,,
\end{equation}
where $\varphi: \mathcal{X} \to \mathcal{H}$ is the feature map, $\{g_m:\mathcal{X}\to\mathbb{R}\}$ are orthonormal eigenfunctions of the integral operator $B: L^2(P_X) \to L^2(P_X)$:
\begin{equation}
B g_m(x) = \int_{\mathcal{X}} k(x,y) g_m(y) dP_X = \lambda_m g_m(x)
\end{equation}
where $P_X$ denotes the marginal density of $X$. We shall use $P_Z$ to denote the marginal density of the instrumental $Z$ to distinguish with the conditional integral operator $T: L^2(P_X)\to L^2(P_Z)$. $\{\lambda_m\}$ are eigenvalues of the covariance operator:
\begin{equation}
\Sigma = \int_{\mathcal{X}} \varphi(x) \otimes \varphi(x) d\mu(x)\,.
\end{equation}
With finite data $\{x_i\}_{i=1}^n$, we approximate:
\begin{align}
K_{ij} &= k(x_i, x_j) = \sum_{m=1}^n \lambda_m \phi_m[i] \phi_m[j]\\
\phi_m[i] &\approx g_m(x_i)\,,
\end{align}
where $\phi_m \in \mathbb{R}^n$ are eigenvectors of the kernel matrix $K$.

\subsection{Order-by-Order Construction}
Recall the perturbative series we constructed 
\begin{equation}
    g(x) = \sum_{i}\sum_{n=0}^\infty \gamma^n\,\alpha^{(n)}(x_i) K(x, x_i) \,,
\end{equation}
was solved iteratively:
\begin{align*}
\text{Order 0:} \quad &(\widetilde{K} + \lambda K)\alpha^{(0)} = h\\
\text{Order 1:} \quad &(\widetilde{K} + \lambda K)\alpha^{(1)} = -3 \sum_{j,k} K_3[i,j,k] \alpha^{(0)}_j \alpha^{(0)}_k\\
\text{Order n:} \quad &(\widetilde{K} + \lambda K)\alpha^{(n)} = \text{RHS}_n\left(\alpha^{(0)}, \ldots, \alpha^{(n-1)}\right)\,,
\end{align*}
where the $n$th order right hand side is a function of various combinations of the previous coefficients. 
\begin{equation}
    \text{RHS}_n = -3\sum_{j,k} K_3[i,j,k]\times
    \left(\alpha_j^{(0)}\alpha_k^{(n-1)}+ \alpha_j^{(1)}\alpha_k^{(n-2)}+\dots+ \alpha_j^{(n-1)}\alpha_k^{(0)}\right)
\end{equation}
and the three-point kernel tensor is:
\begin{equation}
K_3(x_i, x_j, x_k) = \int K(x,x_i) K(x, x_j) K(x, x_k) f(x|Z)dx\,,
\end{equation}
where we use $K[i,j]$ to shorthand entries of the Gram matrix $K(x_i, x_j)$.
% and $\lambda_n = \lambda \times (\text{scale factor})^n$ provides increasing regularization.

In order to see the result spectrally, we project both $\widetilde{K}$ and $K$ onto the eigenbasis of $K$:
\begin{equation}
    M := \Phi^T \widetilde{K} \Phi\,,\quad K = \Phi \Lambda \Phi^T
\end{equation}
where $\Lambda$ is the diagonalized Gram matrix $K$ with eigenvalues $\{\lambda_m\}_{m=1}^n$ and columns of $\Phi$ are the eigenvectors $\phi_m$. The perturbative equations become:
\begin{equation}
    (M+\lambda \Lambda) (\Phi^T\alpha^{(n)}) = \Phi^T\text{RHS}_n
\end{equation}
if $M$ is diagonal, we would jump to conclusion and simply write the spectral version of this equation:
\begin{equation}
    \Phi^T\alpha^{(n)} = \sum_m\frac{\phi_m^T\text{RHS}_n}{M_{mm}+\lambda\lambda_m}
\end{equation}
which in the standard IV kernel ridge regression setting $M_{mm}= \langle T(\sqrt{\lambda_m}\phi_m), T(\sqrt{\lambda_m}\phi_m)\rangle_{L^2(P_Z)}=\eta_m^2 < \lambda_m^2$ (due to the covariance reducing nature of $T$) and $\text{RHS}_0 = h$, which recovers the known expression for the $0$th order kernel ridge regression result:
\begin{equation}
    g^{(0)}(x) = \sum_{m=1}^N\frac{\lambda_m^2}{\eta_m^2+\lambda\lambda_m}\phi^T_m[i] h\,\phi_m[X]
\end{equation}
where spectrally only eigenvectors with larger eigenvalues contribute. 

When $M$ is not diagonal in the IV setting, we whiten the equation using $\Lambda^{1/2}$ as
\begin{equation}
    (\underbrace{\Lambda^{-1/2}M\Lambda^{-1/2}}_{A}+\lambda I) (\Lambda^{1/2}\alpha^{(n)}) = \underbrace{\Lambda^{-1/2}\Phi^T\text{RHS}_n}_{\widetilde{\text{RHS}_n}}
\end{equation}
Diagonalizing the positive semi-definite $A=U\Sigma U^T$ gives
\begin{equation}\label{generic_solution}
    \alpha^{(n)} = \Phi\Lambda^{-1/2} U(\Sigma+\lambda I)^{-1}U^T \,\widetilde{\text{RHS}_n}
\end{equation}
which is shrinking along the eigenvector of $A$ uniformly, not the original kernel eigenvector. Note that the conditional expectation is also, in this sense, inducing spectral mixing by rotating different eigenmodes of the kernel matrix under $U$ depending on how off-diagonal $M$ is. Although we shall see that in addition to this, the three-point tensor introduces even further non-linear spectral mixing. 

In addition, the three-point tensor decomposes as:
\begin{equation}
K_3[i,j,k] 
= \sum_{m,n,p=1}^N \lambda_m \lambda_n \lambda_p \phi_m[i]\phi_m[j] \phi_n[j]\phi_n[k] \phi_p[k]\phi_p[i]\,.
\end{equation}
Define the three-way coupling tensor:
\begin{equation}
T_{mnp}[i,j,k] = \phi_m[i]\phi_m[j] \phi_n[j]\phi_n[k] \phi_p[k]\phi_p[i]\,.
\end{equation}

This creates \textbf{cross-eigenspace interactions}:
\begin{equation}
K_3[i,j,k] = \sum_{m,n,p} \lambda_m \lambda_n \lambda_p T_{mnp}[i,j,k]\,.
\end{equation}
The coupling strength between eigenspaces $m$, $n$, $p$ is weighted by $\lambda_m \lambda_n \lambda_p$.
For the first-order correction:
\begin{align}
\text{RHS}_1[i] &= -\sum_{j,k} K_3[i,j,k] \alpha^{(0)}_j \alpha^{(0)}_k\\
&= -\sum_{m,n,p} \lambda_m \lambda_n \lambda_p \sum_{j,k} T_{mnp}[i,j,k] \alpha^{(0)}_j \alpha^{(0)}_k\,.\nonumber
\end{align}

Since $\alpha^{(0)}_j = \sum_q \alpha^{(0)}_q \phi_q[j]$, this becomes:
\begin{equation}
\text{RHS}_1[i] 
= -\sum_{m,n,p,q,r} \lambda_m \lambda_n \lambda_p \alpha^{(0)}_q \alpha^{(0)}_r \underbrace{\sum_{j,k} T_{mnp}[i,j,k] \phi_q[j] \phi_r[k]}_{I_{mnpqr}}\,.
\end{equation}
The key coupling integral is:
\begin{align*}
I_{mnpqr} &= \sum_{j,k} T_{mnp}[i,j,k] \phi_q[j] \phi_r[k]\\
&= \sum_{j,k} \phi_m[i]\phi_m[j] \phi_n[j]\phi_n[k] \phi_p[k]\phi_p[i] \phi_q[j] \phi_r[k]\\
&= \phi_m[i]\phi_p[i] \left(\sum_j \phi_m[j]\phi_n[j]\phi_q[j]\right)\times\\
&\left(\sum_k \phi_n[k]\phi_p[k]\phi_r[k]\right)\,.
\end{align*}
We shall denote the triple products $\sum_j \phi_a[j]\phi_b[j]\phi_c[j]$ evaluated at the same $X_j$ as $S_{abc}$. $S_{abc}$ depends on the specific kernel one uses, it symbolizes how much higher order interactions a kernel function can have. Hence we have
\begin{equation}
    \text{RHS}_1[i] = -\sum_{m,n,p,q,r}\lambda_m\lambda_n\lambda_p\alpha^{(0)}_q\alpha^{(0)}_r \phi_m[i]\phi_p[i] S_{mnq}S_{npr}\,.
\end{equation}
Using the generic form of the spectral solution \eqref{generic_solution}, we can substitute in the $\text{RHS}_n$ to obtain the corresponding $\alpha^{(n)}$. 
\begin{equation}
    \begin{aligned}
        &g^{(0)}(x) = \sum_{i}K(x,x_i)\Phi\Lambda^{-1/2} U(\Sigma+\lambda I)^{-1}U^T\Lambda^{-1/2}\Phi^T h \\
        &g^{(1)}(x) = \sum_i K(x,x_i)\Phi\Lambda^{-1/2} U(\Sigma+\lambda I)^{-1}U^T\Lambda^{-1/2}\Phi^T\\
        & \sum_{m,n,p,q,r}\lambda_m\lambda_n\lambda_p\alpha^{(0)}_q\alpha^{(0)}_r \phi_m[i]\phi_p[i] S_{mnq}S_{npr}\,.
    \end{aligned}
\end{equation}
Although we can already see that the first order contribution to $g(x)$ contains nonlinear mixing between the eigendirections through $\lambda_m\lambda_n\lambda_p$, it is clearest in the simplified setting of kernel ridge regression. 

In order to see what is happening spectrally clearly, we work under the conditions that $M$ is diagonal or almost diagonal in the Mercer basis, with $M_{mm}=\langle T(\sqrt{\lambda_m}\phi_m), T(\sqrt{\lambda_m}\phi_m)\rangle_{L^2(P_Z)} = \eta_m^2 < \lambda_m^2$, since the conditional operator $T$ is contraction operator reducing the variance. 
Summarizing the zeroth order and first order term:
\begin{equation}
    \begin{aligned}
        &g^{(0)}(x) =\sum_i\sum_{m=1}^N\frac{\lambda_m^2}{\eta_m^2+\lambda\lambda_m}\phi^T_m[i] h\,\phi_m[X]\,; \\
        &g^{(1)}(x)= \sum_i\sum_{m,n,p,q,r}\\&\frac{\lambda_m^3\lambda_n\lambda_p}{\eta_m^2+\lambda\lambda_m}\alpha^{(0)}_q\alpha^{(0)}_r \phi_m[i]\phi_p[i] \phi_m[i]\phi_m[X]S_{mnq}S_{npr}\\
         = & \sum_i\sum_{m,n,p,q,r}\frac{\lambda^3_m\lambda_n\lambda_p}{\eta_m^2+\lambda\lambda_m}\frac{\lambda_q\phi^T_q h}{\eta_q^2+\lambda\lambda_q}\frac{\lambda_r\phi^T_r h}{\eta_r^2+\lambda\lambda_r}\times\\
         &\phi_m[i]\phi_p[i]\phi_m[i]\phi_m[X] S_{mnq}S_{npr}\,.
    \end{aligned}
\end{equation}

Here we introduce the quantity condition number $\kappa$ as the ratio between the largest and smallest eigenvalues of the kernel matrix. Since eigenvalues are labeled from $1$ to $N$ in a descending order for positive definite matrices, we have the following definition:
\begin{equation}
    \kappa = \lambda_1/\lambda_N\,.
\end{equation}

\paragraph{Scale separation!}

Comparing the zeroth and first order solutions, we see that condition number $\kappa$ plays an important role in determining when the first order contribution matters. For simplicity in the arguments, we assume $\eta_m\sim \lambda_m$. Then
\begin{equation}\label{eq:diagnaol_solutions_two_orders}
    \begin{aligned}
        &g^{(0)}(x)=\sum_i\sum_{m=1}^N\frac{\lambda_m}{\lambda_m+\lambda}\phi^T_m[i] h\,\phi_m[X]\,; \\
        &g^{(1)}(x) = \sum_i\sum_{m,n,p,q,r}\frac{\lambda^2_m\lambda_n\lambda_p}{\lambda_m+\lambda}\frac{\phi^T_q h}{\lambda_q+\lambda}\frac{\phi^T_r h}{\lambda_r+\lambda}\times\\
         &\phi_m[i]\phi_p[i]\phi_m[i]\phi_m[X] S_{mnq}S_{npr}\,.
    \end{aligned}
\end{equation}
When $\kappa$ is well-conditioned, $\lambda_m\sim\lambda_p\sim \lambda_q$ are of similar order of magnitude for all $p,q,m$. This simply puts $\alpha^{(1)}$ in the same footing as $\alpha^{(0)}$, roughly counting the number of eigenvalues bigger than the regularization parameter $\lambda$. All eigenvectors contribute equally in this regime. When $\kappa \gg 1$, $\alpha^{(0)}$ only receives contributions from the first few eigenvectors, losing its representation power immensely. However, $\alpha^{(1)}$ is still corrected by further eigenvalues since the coefficient $\lambda_m\lambda_p\lambda_q$ can still remain large when we select $\lambda_m$ for large $m$ as long as we choose $p,q$ to be close to $1$. This introduces additional representation power by mixing in the eigenvectors with small eigenvalues. 
If $\kappa$ is well-conditioned, all eigenfunctions are represented similarly while $\kappa\gg 1$ indicates only the first few eigenfunctions contribute. Kernel functions used in the NPIV scenario has a incredibly large condition number due to the ill-defined convolution integral inversion. We observed that in regular kernel ridge regression, where condition number is small, the perturbative method (addition of the third moment) does not help the zeroth order solution. It is only effective when $\kappa\gg 1$, we shall analyze this phenomenon from the spectral perspective.  

In the case when the condition number is low, all coupling weights $\lambda^2_m \lambda_n \lambda_p \approx \lambda_{\text{typical}}^4$ are similar. $\alpha^{(0)}$ and $\alpha^{(1)}$ both receive ample equal contributions from all the eigenfunctions, so $\alpha^{(1)}$ only adds additional noise to the regressed function. 

However when $\kappa\gg 1$, $\alpha^{(0)}(x_i)K(x, x_i)$ only receives contributions from the eigenfunctions which have large eigenvalues compared to $\lambda$. This is where the triple moment helps, let us analyze case by case: 
\begin{itemize}
\item \textbf{High-to-high coupling}: $m, n, p$ all small $\Rightarrow$ weight $\lambda_m^2 \lambda_n \lambda_p$ is large
\item \textbf{High-to-low coupling}: Mixed indices create medium-strength interactions
\item \textbf{Low-to-low coupling}: $m, n, p$ all large $\Rightarrow$ weight $\lambda_m^2 \lambda_n \lambda_p$ is small
\end{itemize}
We see that $\alpha^{(1)}$ recovers contributions from eigenfunctions with larger eigenvalues as well. This is the manifestation of improving the representation power of the kernel when condition number is high. 

\subsection{The failure of rotationally invariant kernels}
Another issue we can address with the spectral perspective is why rotationally invariant kernels (although having high condition numbers) fit completely trivial lines in high dimensions. The usual kernel ridge regression case is discussed extensively in \cite{donhauser2022fast}, we empirically observe that this is also the case even after adding the third moment term with RBF kernels. The key is the triple eigenfunction integral $S_{abc}$:
\begin{equation}
    S_{abc} = \sum_{i} \phi_a[i]\phi_b[i]\phi_c[i]\,.
\end{equation}
For any rotationally invariant kernel $K(x,y) = k(\|x-y\|)$:
The eigenfunctions can be decomposed into the radial part and spherical harmonics:
\begin{equation}
\phi_{l,m}(x) = R_l(\|x\|) Y_{l,m}(x/\|x\|)\,,
\end{equation}
where $l \in \mathbb{N}_0$ is the radial scaling homogeneity, $m = 1, \ldots, d_l$ where $d_l = \dim(\text{Harm}_l(\mathbb{R}^n))$ and $Y_{l,m}$ are orthonormal spherical harmonics of degree $l$.
The triple product has closed form 
\begin{equation}
\sum_i \phi_a(x_i)\phi_b(x_i)\phi_c(x_i) = \sum_i R_a(r_i)R_b(r_i)R_c(r_i) \times C_{ab}^c(\hat{\omega}_i)\,,
\end{equation}
where $r_i = \|x_i\|$, $\hat{\omega}_i = x_i/\|x_i\|$ and $C_{ab}^c(\hat{\omega}) = Y_a(\hat{\omega})Y_b(\hat{\omega})Y_c(\hat{\omega})$ is the Clebsch-Gordan coefficients. The Clebsch-Gordan coefficients satisfy strict selection rules:
$\langle l_1 m_1, l_2 m_2 | l_3 m_3 \rangle = 0$ unless $|l_1 - l_2| \leq l_3 \leq l_1 + l_2$ and $l_1 + l_2 + l_3 \text{ is even}$ and $m_1 + m_2 + m_3 = 0$.

These would set most of the combinations to $0$, essentially any combination that are not in the representation of the orthogonal group $O(d)$. This explains the failure of RBF kernels even after adding the triple moment term. 

On the other hand for the polynomial kernel $K(x,y) = (x \cdot y + c)^d$ on $\mathbb{R}^n$, eigenfunctions are multinomial in the coordinates:
\begin{equation}
\varphi_a(x) = \sqrt{\lambda_a} \prod_{j=1}^n x_j^{a_j} \times \text{(normalization factor)}\,,
\end{equation}
where $a = (a_1, \ldots, a_n)$ is a multi-index with $|a| = a_1 + \cdots + a_n \leq d$.
The triple product is simply
\begin{align}
&\sum_i \varphi_a(x_i)\varphi_b(x_i)\varphi_c(x_i) \\
&= \text{(normalization)} \times \sum_i \prod_{j=1}^n x_{i,j}^{a_j + b_j + c_j} \\
&= \text{(normalization)} \times \sum_i \prod_{j=1}^n x_{i,j}^{d_j}\,,
\end{align}
where $d_j = a_j + b_j + c_j$.
This is just a power sum in each coordinate which does not impose any selection rule.

%%%%%%%%%%%%%%%%%%%%%%%%%%%%%%%%
\subsection{The need for renormalization}\label{subsec:need_for_renorm}
In this subsection, we perform some simple counting using the usual kernel ridge regression case to demonstrate the inevitable divergence of the naive perturbative series. In the simplest setting where $\widetilde{K}$ is diagonalizable in the kernel eigenbasis \eqref{eq:diagnaol_solutions_two_orders}, from the recursive structure, we can derive that:
\begin{equation}
\alpha^{(n)}_m \sim \frac{\lambda_m^{3n}}{\left(\eta_m + \lambda\right)^{2n+1}} \times \text{geometric factors}\,.
\end{equation}

This leads to the crucial scaling:
\begin{align}
\text{Order 0:} \quad |\alpha^{(0)}_m| &\sim \frac{1}{\eta_m + \lambda}\\
\text{Order 1:} \quad |\alpha^{(1)}_m| &\sim \frac{\lambda_m^3}{(\eta_m + \lambda)^3}\\
\text{Order 2:} \quad |\alpha^{(2)}_m| &\sim \frac{\lambda_m^6}{(\eta_m + \lambda)^5}\\
\text{Order n:} \quad |\alpha^{(n)}_m| &\sim \frac{\lambda_m^{3n}}{(\eta_m + \lambda)^{2n+1}}\,.
\end{align}
We shall assume $\lambda_m\sim \eta_m$ for simplicity.

\paragraph{Well-Conditioned Case ($\kappa$ not divergent)}

When all eigenvalues are similar: $\lambda_m \approx \lambda_{\text{typical}}$ for all $m$. Combining with the coupling $\gamma$:
\begin{equation}
    \frac{\gamma^n|\alpha^{(n)}|}{\gamma^{n-1}|\alpha^{(n-1)}|} \sim \gamma\,\frac{\lambda_{\text{typical}}^3}{(\lambda_{\text{typical}}+\lambda)^2} \begin{cases}
        \lambda_{\text{typical}}\gg \lambda : \gamma<1\\
        \lambda_{\text{typical}} \sim \lambda: \gamma < \frac{4}{\lambda}\\
        \lambda_{\text{typical}} \ll \lambda : \text{always convergent}
    \end{cases}
\end{equation}
We listed the convergence requirement for $\gamma$, as long as we make appropriate choices of $\lambda$ compatible with the ridge regularization parameter, the series will be convergent. Hence we would not need any renormalization anyways since all contributions vanish at high order. The perturbative series terminates naturally, but provides no improvement over zeroth order.

\paragraph{Ill-Conditioned Case ($\kappa \gg 100$)}

Eigenvalues have extreme hierarchy: few large $\lambda_1\gg \lambda_2 \gg \dots \gg \lambda \gg \dots \gg \lambda_N$. The convergence condition for the series becomes:
\begin{equation}
    \frac{\gamma^n|\alpha^{(n)}|}{\gamma^{n-1}|\alpha^{(n-1)}|} \sim\gamma\,\frac{\lambda_m\lambda_p\lambda_q}{(\lambda_p+ \lambda)^2}\,,
\end{equation}
which can go to infinity or $0$ depending on the eigenvalues chosen. The rescaling renormalization we used 
\begin{equation}
    \gamma_n = \gamma^n \times \text{min}\left(\frac{|\alpha^{(0)}|}{|\alpha^{(n)}|}, 1\right)\,,
\end{equation}
automatically implements spectral flow:
\begin{itemize}
\item \textbf{Stable eigenspaces}: Contribute little to $\|\alpha^{(n)}\|$ $\Rightarrow$ remain at full strength
\item \textbf{Unstable eigenspaces}: Dominate $\|\alpha^{(n)}\|$ $\Rightarrow$ get suppressed by small $\mathcal{R}_n$
\item \textbf{Critical eigenspaces}: Drive the renormalization flow
\end{itemize}

%%%%%%%%%%%%%%%%%%%%%%%%%%%%%%%%%%%
\section{Background on resurgence}\label{appendix:resurgence}
Using the mathematical tool of resurgence and Borel resummation \cite{ecalle1981fonctions}, we come up with a further enhancement of the renormalized perturbative series. Note that the series (even after renormalization) is still an asymptotic series. In standard QFT context, this is due to non-perturbative effects such as field configurations that are trivial \cite{Berry1991,Mari_o_2014resurgence, dorigoni2019resurgence,AIHPA_1999__71_1_1_0resurgence, costin2023goingresurgentbridge}, in our case, these are solutions that lie in the null space of the kernel matrix. They are trivialized in the perturbative series yet contributes to the solution space of the ill-defined inverse integral equation. 

The prevalence of saddle points in high dimensional optimization problems is the key issue here. Several facts about them that makes the usual methods fail:
\begin{itemize}
    \item they are never correctly captured by finite truncations of any perturbative series.
    \item they are minimizer of the action/loss function that are "trivial" in the sense that they are not solutions to the perturbative equation of motion truncated at any order. 
    \item They live in the kernel of the perturbative operator, which is difficult to compute and usually the cause of asymptotic series in the perturbative solution.
\end{itemize}
Luckily, Borel resummation is to rescue, it provides a systematic way of analyzing these saddles and incorporate them in the solution. 

You also run into these issues in neural networks, where they appear as saddles of the loss function, which grow in number as the dimension of the loss landscape increases. This is usually tackled using optimizers and schedulers to create local perturbations that destabilizes the saddle configuration. 

\subsection{Introduction to resurgence}
\begin{definition}[Asymptotic Series]
A formal power series $\sum_{n=0}^{\infty} a_n z^n$ is called an asymptotic expansion of $f(z)$ as $z \to 0$ in a sector $S$ if for every $N \geq 0$:
\begin{equation}
f(z) - \sum_{n=0}^{N-1} a_n z^n = O(z^N) \quad \text{as } z \to 0 \text{ in } S\,.
\end{equation}
We write $f(z) \sim \sum_{n=0}^{\infty} a_n z^n$.
\end{definition}

\begin{example}[The Euler Integral]
Consider the integral representation:
\begin{equation}
E(z) = \int_0^{\infty} \frac{e^{-t}}{1 + zt} dt, \quad |\arg(z)| < \frac{\pi}{2}\,.
\end{equation}

Integration by parts yields the asymptotic expansion:
\begin{equation}
E(z) \sim \sum_{n=0}^{\infty} (-1)^n n! z^n \quad \text{as } z \to 0\,.
\end{equation}

This series diverges for all $z \neq 0$ since $a_n = (-1)^n n!$ grows factorially, yet it provides the unique asymptotic expansion of $E(z)$.
\end{example}

For factorially divergent series $\sum a_n z^n$ with $|a_n| \sim n!/r^n$, the optimal truncation occurs at:
\begin{equation}
N^* \approx \frac{r}{|z|}
\end{equation}

The truncation error is exponentially small:
\begin{equation}
\left|f(z) - \sum_{n=0}^{N^*} a_n z^n\right| \sim e^{-r/|z|}\,.
\end{equation}

This exponential accuracy is the hallmark of asymptotic series, but it also reveals their fundamental limitation: exponentially small terms are completely invisible.

\begin{definition}[Borel Transform]
For a formal power series $\phi(z) = \sum_{n=0}^{\infty} a_n z^n$, the Borel transform is:
\begin{equation}
\mathcal{B}[\phi](\zeta) = \sum_{n=0}^{\infty} \frac{a_n}{n!} \zeta^n\,.
\end{equation}
\end{definition}

\begin{remark}
If $|a_n| \sim n!$, then the Borel transform typically has finite radius of convergence, converting a divergent series into a convergent one.
\end{remark}

When $\mathcal{B}[\phi](\zeta)$ can be analytically continued to the positive real axis, we can define:

\begin{definition}[Borel Sum]
The Borel sum of $\phi(z)$ is:
\begin{equation}
\mathcal{S}[\phi](z) = \frac{1}{z} \int_0^{\infty} e^{-\zeta/z} \mathcal{B}[\phi](\zeta) d\zeta\,,
\end{equation}
provided the integral converges.
\end{definition}

\begin{example}[Borel Resummation of Euler Integral]
For $E(z) \sim \sum_{n=0}^{\infty} (-1)^n n! z^n$, the Borel transform is:
\begin{equation}
\mathcal{B}[E](\zeta) = \sum_{n=0}^{\infty} \frac{(-1)^n n!}{n!} \zeta^n = \sum_{n=0}^{\infty} (-\zeta)^n = \frac{1}{1+\zeta}\,.
\end{equation}

The Borel sum gives:
\begin{equation}
\mathcal{S}[E](z) = \frac{1}{z} \int_0^{\infty} \frac{e^{-\zeta/z}}{1+\zeta} d\zeta = E(z)\,.
\end{equation}

Thus, Borel resummation exactly recovers the original function!
\end{example}

With a subtlety that we have secretly changed the position of an integral and a sum, which might not be convergent. For $x<< 1$:
\begin{example}
    \begin{equation}
    \sum_{n=0}^\infty n! x^n = \sum_{n=0}^\infty x^n \int_{0}^\infty dt\, e^{-t}\, t^n = \int_{0}^\infty e^{-t}\,\frac{1}{1-xt}\,.
    \end{equation}
\end{example}
This integral is singular on the positive real axis at $t=\frac{1}{x}$. This forces us to analytically continue and change the contour. There are naturally two choices to do that, which results in an ambiguity in the result. They differ by the residue of the integrand at the singularity. Surprisingly, this residue called Stokes constant also corresponds to the saddle point approximation of the integral representation of the series. This is what we rely on to do the analysis on our series phrased as a minimization problem in high dimensions with multiple saddle points. 

\begin{example}
    We can also show that if an integral representation exists for the equation we are trying to solve, the saddle point approximation of the integral is exactly what the residue at the poles on Borel plane of the Laplace transform is computing. For the Airy equation:
    \begin{equation}
        \frac{d^2 y}{dx^2} - xy = 0\,,
    \end{equation}
    perturbatively, it has the following solution:
    \begin{equation}
        y(x) = \sum_{n=0}^\infty \left(-\frac{3}{4}\right)^n\,\frac{\Gamma(n+5/6)\Gamma(n+1/6)}{n!}\,(x^{-3/2})^n\,.
    \end{equation}
    Its Borel resummation has a singularity at $t = -\frac{2}{3}x^{3/2}$. The residue at the singularity is 
    \begin{equation}
        \text{Res}_{t=-\frac{2}{3}x^{3/2}} \boldsymbol{B}y[t] = e^{2/3 x^{3/2}}\,.
    \end{equation}
    Then we do a saddle point approximation for the Airy function's integral representation:
    \begin{equation}
        y(x) = \int_{0}^\infty e^{t^3/3 - xt}dt  \sim \left.e^{t^3/3-xt}\right\vert_{t=\sqrt{x}}\sqrt{\frac{2\pi i}{1}} = e^{2/3 x^{3/2}} \,,
    \end{equation}
    which precisely matches the residue at the singularity on the Borel plane. Usually it is difficult to solve for the saddle point of the minimizer in high dimensions, we can use the Borel analysis to bypass that. 
\end{example}

What was mentioned here is a generic pattern that we can formulate and prove as a theorem \cite{Bhattacharya:2024hhh}:
\begin{theorem}
    The critical point contributions to the integral representation of an asymptotic series are in one-to-one correspondence with the poles on Borel plane. 
\end{theorem}
\textit{proof:} A simple proof for this involves a trick in measure theory named the co-area formula, which is used to reduce the dimension of an integral onto some lower dimensional domain level set. 
\begin{equation}
    \int_{\Omega\in\mathbb{R}^n} g(x) |J_k u(x)| dx = \int_{\mathbb{R}^k} \left(\int_{u^{-1}(t)} g(x)dH_{n-k}(x)\right)\, dt\,,
\end{equation}
where $k<n$, $u(x)=t$ labels the $n-k$ dimensional level set, $|J_k u(x)|$ is its $k$-dim Jacobian and $dH_{n-k}$ represents the appropriate Hausdorff measure on the level set. Multiplying both sides with a delta function to indicate the level set $\delta(t-u(x))$. Then we have
\begin{equation}
    \int_{\Omega} g(x) \delta(t-u(x)) |\nabla u(x)| dx = \int_{t=u(x)} g(x) \,dH(x)\,.
\end{equation}
Now choosing $g(x) = \frac{1}{|\nabla u(x)|}$ gives 
\begin{equation}\label{co-area-trick}
    \int_{\Omega} \delta(t-u(x)) = \int_{t= u(x)} \frac{d H(x)}{|\nabla u(x)|} \,.
\end{equation}
We shall use this formula in the following proof. 

Writing the multivariant function as an asymptotic series $f(g) = \sum_{n=0}^{\infty}a_n g^n$. We denotes its Borel transform and the corresponding inverse transform:
\begin{equation}
    B[f](t) = \sum_{n=0}^{\infty} \frac{a_n}{n!} (t)^n\,.
\end{equation}
And 
\begin{equation}\label{eq:compare}
    \mathcal{B}(B[f])(g) = \frac{1}{g}\int_0^\infty e^{-t/g} B(t)\,.
\end{equation}
If an integral representation also exists:
\begin{equation}
    f(g) = \int e^{-S(x)/g} \,d x\,.
\end{equation}
Then we can write it on a level set then integrate over the choice of level:
\begin{equation}
    f(g) = \frac{1}{g}\int_0^{\infty} d t\, e^{-t/g} \left(g\int d x\,\delta(t-S(x))\right)\,,
\end{equation}
where we have already replaced the exponentiated $S(x)$ with $t$. Using the co-area trick \eqref{co-area-trick}, we have
\begin{equation}
    f(g) = \frac{1}{g}\int_{0}^\infty dt\, e^{-t/g} \,\left(g\int_{t=S(x)} \frac{d \sigma(x)}{|\nabla S(x)|}\right)\,,
\end{equation}
where $d\sigma(x)$ is the appropriate measure. Comparing this with the Laplace transform formula \eqref{eq:compare}. We see that the Borel transform can be written as 
\begin{equation}
  B\left[\frac{1}{g}\,f\right](t) = \int_{t=S(x)} \frac{d \sigma(x)}{|\nabla S(x)|} \,.
\end{equation}
It is easy to see that singularities of the Borel transform occur precisely when $\nabla S(x)=0$, which are critical points of the integral representation.  This is a fact that can be extremely useful for more general usage of resurgence in neural networks\footnote{To be elaborated in a separate section, off the the scope of this project.}. 

A further way to express the Borel transform derived from this is 
\begin{equation}
    B[f(g)](t) = \int d^n x \,\Theta(t-S(x))\,,
\end{equation}
where we see Borel transform of $f$ is the volume of domain for which $S(x)\leq t$. There are three possible sources of divergence or singularity on Borel plane. 
\begin{itemize}
    \item S(x) is critical at $S(x)=t_*$
    \item S(x) goes to some finite $t_*$ as $x$ goes to infinity, the volume integral diverges
    \item At some $t\to t_{*}$, $S(x)$ has infinite volume.
\end{itemize}

\subsection{The algorithm}
We include the detailed algorithm for resurgence \ref{algo:resurgence} and full implementation of renormalization of perturbative series here, which asymptotically has a time complexity of $\mathcal{O}(n^2)$, with $n$ data points.
\begin{algorithm}
\caption{Perturbative NPIV with Renormalization}
\begin{algorithmic}\label{algo:renormalization}
\STATE \textbf{Input:} Data $(X, Z, Y)$, kernel function $K$, ridge regularization coefficient $\lambda$, base perturbation parameter $\gamma$, max order $M$, $\epsilon=10^{-4}$ 
\STATE \textbf{Output:} Coefficient vectors $\alpha^{(i)}$

\STATE Compute from dataset $K_{ij}$, $h_i$, $K_{ijk}$
\STATE Solve $(\widetilde{K} + \lambda K)\alpha^{(0)} = h$ for zeroth-order solution
\STATE $\|\alpha^{(0)}\| \leftarrow$ norm of zeroth-order coefficients

\FOR{$m = 1$ to $M$}
    \STATE Compute right-hand side: $rhs_i^{(m)} = -3 \sum_{j,k} K_{ijk}(\alpha^{(0)}_j\alpha^{(m-1)}_k+\dots+\alpha^{(m-1)}_j\alpha^{(0)}_k)$
    \STATE Solve $(\widetilde{K} + \lambda K)\alpha^{(m)} = rhs^{(m)}$
    \STATE $\|\alpha^{(m)}\| \leftarrow$ norm of $m$-th order coefficients
    
    \IF{renormalization needed}  \STATE Compute adaptive gamma: $\gamma_m = \gamma^m \times \min(1, \|\alpha^{(0)}\| / \|\alpha^{(m)}\|)$
        \IF{$\|\alpha^{(0)}\| / \|\alpha^{(m)}\| < \epsilon$}
            \STATE $\gamma_m = 0$ \COMMENT{Zero out correction if too unstable}
        \ENDIF
    \ELSE \STATE $\gamma_m = \gamma^m$ \COMMENT{Use standard perturbation parameter}
    \ENDIF
\ENDFOR

\STATE \textbf{return} $\{\alpha^{(n)}\}_{n=0}^M$
\end{algorithmic}
\end{algorithm}

\begin{algorithm}
\caption{Resurgence algorithm for asymptotic series}
\begin{algorithmic}\label{algo:resurgence}
\STATE \textbf{Input:} the raw perturbative series $\{\alpha^{(i)}\}_{i=0}^M$, $\{\gamma_i\}_{i=1}^M$

\IF{Not asymptotic series:}
    \STATE \textbf{return} renormalised series: $g(x_{new}) = \sum_{n=0}^M \sum_i \gamma_n \alpha_i^{(n)} K(x_{new}, x_i)$
\ELSE
    \FOR{$i = 1$ to $|\alpha|$}
        \STATE Compute Borel transform: $\hat{\alpha}_i(s) = \sum_{n=0}^M \frac{\alpha_i^{(n)}}{n!} s^n$
    \ENDFOR
    \STATE Construct Borel transformed function: $\hat{g}(tX) = \sum_i \hat{\alpha}_i(\gamma t) K(tX, X_i)$

    \STATE Initialize singularity set: $\mathcal{S} = \emptyset$
    \FOR{$t$ on complex contour around origin}
        \IF{$|\hat{g}(tX)| > \frac{1}{\delta}$ or $\hat{g}(tX)$ undefined}
            \STATE $\mathcal{S} \leftarrow \mathcal{S} \cup \{t\}$
        \ENDIF
    \ENDFOR
    \STATE Let $\{t_k\}_{k=1}^m$ be the dominant singularities from $\mathcal{S}$
    
    \FOR{$k = 1$ to $m$}
        \STATE Compute Stokes constant via residue theorem:
        \STATE $C_k \leftarrow \text{Res}_{t=t_k} \hat{g}(tX)$ \COMMENT{Residue at singularity $t_k$}
        
        \FOR{$i = 1$ to $|\alpha|$}
            \STATE $g_{k,i}(X) \leftarrow$ analytic continuation of $\hat{\alpha}_i$ around $t_k$
        \ENDFOR
    \ENDFOR
    
    \FOR{prediction point $x_{new}$}
        \STATE Base perturbative series: $g_{pert}(x_{new}) \leftarrow \sum_{n=0}^M \sum_i \gamma_n \alpha_i^{(n)} K(x_{new}, x_i)$
        
        \STATE Trans-series corrections: 
        \STATE $g_{trans}(x_{new}) \leftarrow \sum_{k=1}^m C_k \exp\left(-k \frac{t_k(x_{new})}{x_{new}}\right) \sum_i g_{k,i}(x_{new}) K(x_{new}, x_i)$
        
        \STATE Combined prediction: $\Phi(x_{new}) \leftarrow g_{pert}(x_{new}) + g_{trans}(x_{new})$
    \ENDFOR
    \STATE \textbf{return} $\Phi(X)$
\ENDIF

\STATE \textbf{return} $\{\alpha^{(i)}\}_{i=0}^M$, $\{\gamma_i\}_{i=1}^M$, $\Phi(x)$
\end{algorithmic}
\end{algorithm}
%%%%%%%%%%%%%%%%%%%%%%%%%%%%%

%%%%%%%%%%%%%%%%%%%%%%%%%%%%%%%%%%%%
\section{The irrelevance of physical renormalization}\label{appendix:irrelevance_of_physical_renormalization}
In physics, the Wilsonian paradigm of renormalization is based on the following philosophy. If there is a scale $\Lambda$ one introduces in a physical system beyond which the degrees of freedom is not considered. The physical observables of the system should not depend on the magnitude of $\Lambda$. This requirement creates a spectral flow automatically called the renormalization group flow, adjusting the coupling constant accordingly. In our case, there is a beautiful analogy. Recall the series ansatz for the regressed solution:
\begin{equation}
    g(x) = \sum_{n=0}^\infty\sum_{i} \gamma^n\alpha^{(n)}(X_i) K(X_i, X)\,.
\end{equation}
We shall argue that the regularization scale $\lambda$ automatically serves as the cut-off scale $\Lambda$. To see this, we define the following quantity:
\begin{equation}
    \rho = \frac{\lambda_m}{\lambda_m + \lambda}\,,
\end{equation}
which when $\lambda \gg \lambda_m$, $\rho\to 0$ and when $\lambda \ll \lambda_m$, $\rho\to 1$. Hence it roughly is measuring the amount of eigenvalues larger than $\lambda$, the ones below $\lambda$ do not really contribute. Taking the target function $g(X)$, a physical renormalization process is triggered by recognizing that $g(X)$ should not depend on the choice of $\lambda$\footnote{This is not quite the case in regression, since $\lambda$ controls the bias/variance trade off, its scale physically impacts the regressed result.}. Although it is intriguing to make the exact identification with physics, it is not quite correct due the aforementioned reason. We think it is nevertheless important to include this in the script.

Nevertheless, if we follow this logic, we see that scale independence of $g(X)$ implies the Euler vector action:
\begin{equation}
    \lambda\frac{d}{d\lambda}\,g(X) = 0\,.
\end{equation}
This can not be achieved if $\gamma$ is independent of $\lambda$, hence this physical requirement and the existence of imbalanced condition number creates a run on the coupling $\gamma$, demanding it to depend on $\lambda$. More specifically, after some computations, one finds:
\begin{equation}
    \beta(\lambda)= \lambda\frac{d}{d\lambda} = 
    -\frac{\sum_n\gamma_n\sum_m\frac{\lambda}{\lambda_m + \lambda}\left(\partial_\lambda(\phi_m^T\,\text{RHS}_n)-\alpha^{(n)}_m(\lambda)\right)\phi_m}{\sum_n \gamma_{n-1}\sum_m\alpha^{(n)}_m(\lambda)\phi_m}\,,
\end{equation}
where all regimes of $\lambda_m$ contributes. Our method of rescaling the coupling $\gamma$ based on the norm ratio of the $\alpha^{(n)}$s is inspired by the physical method of renormalization.

%%%%%%%%%%%%%%%%%%%%%%%%%%%%%%%%%%%%
\section{Generalized Representer Theorem for Perturbative NPIV}\label{appendix:representer_theorem_generalized}

We extend the classical representer theorem to accommodate higher-order interaction terms. This section provides a derivation of why our series expansion method maintains a finite-dimensional representation at each order, even with cubic interaction terms.

\paragraph{Problem Formulation and Notation}

Let us begin by precisely defining the mathematical framework:

\begin{definition}[Reproducing Kernel Hilbert Space]
A Reproducing Kernel Hilbert Space (RKHS) $\mathcal{H}$ with kernel $K: \mathcal{X} \times \mathcal{X} \rightarrow \mathbb{R}$ is a Hilbert space of functions $f: \mathcal{X} \rightarrow \mathbb{R}$ such that:
\begin{enumerate}
    \item For all $x \in \mathcal{X}$, the function $K(\cdot, x)$ belongs to $\mathcal{H}$.
    \item For all $x \in \mathcal{X}$ and all $f \in \mathcal{H}$, the reproducing property holds: $f(x) = \langle f, K(\cdot, x) \rangle_{\mathcal{H}}$.
\end{enumerate}
\end{definition}

\begin{definition}[Conditional Expectation Operators]
Let $X$ and $Z$ be random variables with joint distribution $P_{X,Z}$. We define:
\begin{enumerate}
    \item The conditional expectation operator $T: \mathcal{H} \rightarrow L^2(P_Z)$ as $(Tf)(z) = \mathbb{E}[f(X)|Z=z]$.
    \item The adjoint operator $T^*: L^2(P_Z) \rightarrow \mathcal{H}$ satisfies $\langle Tf, g \rangle_{L^2(P_Z)} = \langle f, T^*g \rangle_{\mathcal{H}}$ for all $f \in \mathcal{H}$ and $g \in L^2(P_Z)$.
\end{enumerate}
\end{definition}

We consider the nonparametric instrumental variable (NPIV) problem with a cubic interaction term:
\begin{align}
S[f] &= S_0[f] + \gamma S_1[f]\\
&=\mathbb{E}_Z\left[\left(\mathbb{E}[Y|Z] - \mathbb{E}[f(X)|Z]\right)^2\right] + \lambda\|f\|^2_{\mathcal{H}} + \gamma S_1[f]\,,
\end{align}
where $S_1[f]$ represents the non-independent three-point interaction:
\begin{align}
S_1[f] = \frac{2}{3}\mathbb{E}_Z\left[\mathbb{E}\left[f(X)^3 \,\bigg|\, Z\right]\right]\,.
\end{align}
The expectation over $Z$ (notated as $\mathbb{E}_Z$) is needed since the moment condition $\mathbb{E}[Y - g(X) | Z] = 0$ requires considering all possible values of $Z$, it is also what we compute in the acutal implementation.

Our goal is to find the function $f^* \in \mathcal{H}$ that minimizes $S[f]$. We approach this using a perturbative expansion:
\begin{align}
f(x) = f_0(x) + \gamma f_1(x) + \gamma^2 f_2(x) + \ldots\,.
\end{align}

\paragraph{Functional Derivatives and Variational Calculus}

To apply variational methods, we must understand functional derivatives:

\begin{definition}[Functional Derivative]
The functional derivative of a functional $S[f]$ with respect to the function $f$ at point $x$, denoted $\frac{\delta S[f]}{\delta f(x)}$, represents the rate of change of $S[f]$ when $f$ is perturbed infinitesimally at the point $x$:
\begin{align}
\frac{\delta S[f]}{\delta f(x)} = \lim_{\epsilon \to 0} \frac{S[f + \epsilon\delta_x] - S[f]}{\epsilon}\,,
\end{align}
where $\delta_x$ is the Dirac delta function centered at $x$.
\end{definition}

% \begin{lemma}[Functional Derivatives of Common Terms]
% The following functional derivatives will be useful in our analysis:
% \begin{align}
% \frac{\delta}{\delta f(x)}\int f(y)g(y)dy &= g(x)\\
% \frac{\delta}{\delta f(x)}\int f(y)^2g(y)dy &= 2f(x)g(x)\\
% \frac{\delta}{\delta f(x)}\int f(y)^3g(y)dy &= 3f(x)^2g(x)\\
% \frac{\delta}{\delta f}\|f\|^2_{\mathcal{H}} &= 2f\,.
% \end{align}
% \end{lemma}

The adjoint operator $T^*$ plays a crucial role in our derivations:
\begin{proposition}[Adjoint of $T$ in an RKHS]
For $(Tf)(z)=\E[f(X)\mid Z=z]$ with RKHS kernel $K$, the adjoint $T^*:L^2(P_Z)\to\mathcal{H}$ is
\[
T^* g \;=\; \E\!\big[g(Z)\,K(\cdot,X)\big].
\]
\end{proposition}

\begin{proof}
For any $f\in\mathcal{H}$ and $g\in L^2(P_Z)$,
\begin{equation}
\langle Tf, g\rangle_{L^2(P_Z)}
= \E\!\big[\E[f(X)\mid Z]\;g(Z)\big]
= \E\!\big[\langle f,K(\cdot,X)\rangle_{\mathcal{H}}\;g(Z)\big]
= \langle f,\,\E[g(Z)\,K(\cdot,X)]\rangle_{\mathcal{H}}.
\end{equation}
\end{proof}

%%%%%%%%%%%%%%%%%%%%%%%%%%%%%%%%
\begin{proposition}[Explicit Form of the Adjoint Operator]
For the conditional expectation operator $T$ defined above, the adjoint operator $T^*: L^2(P_Z) \rightarrow \mathcal{H}$ can be explicitly written as:
\begin{align}
(T^*g)(x) = \mathbb{E}[g(Z)|X=x]\,.
\end{align}
\end{proposition}

\begin{proof}
For any $f \in \mathcal{H}$ and $g \in L^2(P_Z)$, we have:
\begin{align}
\langle Tf, g \rangle_{L^2(P_Z)} &= \int (Tf)(z)g(z)dP_Z(z)\\
&= \int \mathbb{E}[f(X)|Z=z]g(z)dP_Z(z)\\
&= \iint f(x)dP_{X|Z}(x|z)g(z)dP_Z(z)\\
&= \iint f(x)g(z)dP_{X,Z}(x,z)\\
&= \int f(x) \left(\int g(z)dP_{Z|X}(z|x)\right)dP_X(x)\\
&= \int f(x)\mathbb{E}[g(Z)|X=x]dP_X(x)\\
&= \langle f, T^*g \rangle_{\mathcal{H}}
\end{align}
Thus, $(T^*g)(x) = \mathbb{E}[g(Z)|X=x]$.
\end{proof}

\paragraph{Zeroth-Order Solution}

For $\gamma = 0$, we consider the standard regularized NPIV problem:

\begin{equation}
\min_{f \in \mathcal{H}} S_0[f] = \min_{f \in \mathcal{H}} \left\{\mathbb{E}_Z\left[\left(\mathbb{E}[Y|Z] - \mathbb{E}[f(X)|Z]\right)^2\right] + \lambda\|f\|^2_{\mathcal{H}}\right\}\,.
\end{equation}

Using operator notation, we can write this as:
\begin{align}
S_0[f] = \|h - Tf\|_{L^2(P_Z)}^2 + \lambda\|f\|^2_{\mathcal{H}}\,,
\end{align}
where $h(z) = \mathbb{E}[Y|Z=z]$.

\begin{theorem}[Zeroth-Order Representer Theorem]
The minimizer $f_0$ of $S_0[f]$ has the form:
\begin{align}
f_0(x) = \sum_{i=1}^n \alpha_i^{(0)} K(x, x_i)\,,
\end{align}
where $\{x_1, x_2, \ldots, x_n\}$ are the observed data points. The coefficients $\alpha^{(0)}$ are the solution to:
\begin{align}
(\widetilde{K} + \lambda K)\alpha^{(0)} = h\,,
\end{align}
where $\widetilde{K}_{ij} = \E[\int\int K(x,x_i) K(x',x_j)f(x|Z)f(x'|Z)dxdx']$, $K_{ij} = K(x_i, x_j)$ and $h_i = (T^*h)(x_i)$.
\end{theorem}

\begin{proof}
Taking the functional derivative of $S_0[f]$ and setting it to zero:
\begin{align}
\frac{\delta S_0[f]}{\delta f(x)} &= -\langle h - Tf, T\delta_x \rangle_{L^2(P_Z)} + \lambda \langle f, \delta_x \rangle_{\mathcal{H}}\\
&= -(T^*(h - Tf))(x) + \lambda f(x)\,,
\end{align}
where we've used the definition of the adjoint operator $T^*$. Setting this to zero:
\begin{align}
(T^*(h - Tf))(x) = \lambda f(x)
\end{align}

By the standard representer theorem argument decomposing the function in basis in the RKHS and orthogonal, minimizing such objective easily gives us the form of the solution, the solution has the form $f_0(x) = \sum_{i=1}^n \alpha_i^{(0)} K(x, x_i)$. Substituting this and evaluating at a finite collection of data points $\{x_i\}_{i=1}^n$ gives the linear system:
\begin{align}
(\widetilde{K} + \lambda K)\alpha^{(0)} = h\,,
\end{align}
where $h_i = (T^*h)(x_i) = \mathbb{E}[Y\,K(X,x_i)|Z]$.
\end{proof}

\paragraph{First-Order Correction}

For the first-order correction, we need to find $f_1$ by expanding $S[f]$ around $f_0$:
\begin{equation}
S[f_0 + \gamma f_1] =
S_0[f_0] + \gamma \left( \frac{\delta S_0[f]}{\delta f}\bigg|_{f=f_0} \cdot f_1 + S_1[f]\bigg|_{f=f_0} \right) + O(\gamma^2)\,.
\end{equation}
Setting the variational derivative with respect to $f_1$ to zero, using
\begin{align}
\frac{\delta S_0[f]}{\delta f}\bigg|_{f=f_0} = T^*(h - Tf_0) + \lambda f_0\,.
\end{align}
Using operator notation, the equation for $f_1$ becomes:
\begin{align}
(T^*T + \lambda I)f_1 = -\frac{\delta S_1[f]}{\delta f}\bigg|_{f=f_0}\,.
\end{align}
Then for the second term $S_1[f] = \frac{2}{3}\mathbb{E}_Z[\mathbb{E}[f(X)^3|Z]]$:
\begin{align}
\frac{\delta S_1[f]}{\delta f(x)} &= \frac{\delta}{\delta f(x)}\left(\frac{2}{3}\mathbb{E}_Z[\mathbb{E}[f(X)^3|Z]]\right)\\
&= \frac{2}{3}\mathbb{E}_Z\left[\frac{\delta}{\delta f(x)}\mathbb{E}[f(X)^3|Z]\right]\,.
\end{align}
For any function $g(X)$, the functional derivative with respect to $f(x)$ is:
\begin{align}
\frac{\delta g(X)}{\delta f(x)} = \frac{\partial g(X)}{\partial f(X)}\frac{\delta f(X)}{\delta f(x)} = \frac{\partial g(X)}{\partial f(X)}\delta(X-x)\,.
\end{align}
Therefore:
\begin{align}
\frac{\delta}{\delta f(x)}\mathbb{E}[f(X)^3|Z] &= \mathbb{E}\left[\frac{\delta f(X)^3}{\delta f(x)}\bigg|Z\right]\\
&= \mathbb{E}\left[3f(X)^2\delta(X-x)\bigg|Z\right]\,.
\end{align}
The presence of the delta function $\delta(X-x)$ inside the expectation requires careful treatment. For a continuous random variable $X$ with density $p(X|Z)$:
\begin{align}
\mathbb{E}[\delta(X-x)g(X)|Z] &= \int \delta(y-x)g(y)p(y|Z)dy\\
&= g(x)p(x|Z)\,.
\end{align}
Thus:
\begin{align}
\mathbb{E}[3f(X)^2\delta(X-x)|Z] = 3f(x)^2p(x|Z)\,.
\end{align}
Substituting back:
\begin{align}
\frac{\delta S_1[f]}{\delta f(x)} &= \frac{2}{3}\mathbb{E}_Z[3f(x)^2p(x|Z)]\\
&= 2f(x)^2\mathbb{E}_Z[p(x|Z)]\\
&= 2f(x)^2p(x)\,,
\end{align}
where $p(x)$ is the marginal density of $X$.

However, in our RKHS setting, we work with functions rather than explicitly with densities. In particular, when evaluating at our data points, we can use the connection to the adjoint operator and kernel functions.

Evaluating the equation at a data point $x_i$:
\begin{align}
(T^*T + \lambda I)f_1(x_i) &= -f_0(x_i)^2p(x_i)\,.
\end{align}
Substituting the representer form of $f_0$:
\begin{equation}
f_0(x_i)^2 = \left(\sum_{j=1}^n \alpha_j^{(0)} K(x_i, x_j)\right)^2
= \sum_{j=1}^n\sum_{k=1}^n \alpha_j^{(0)}\alpha_k^{(0)} K(x_i, x_j)K(x_i, x_k)\,.
\end{equation}
When we apply this in the RKHS context, the term $p(x_i)$ is effectively handled by kernel methods. In particular, the right-hand side of our equation becomes:
\begin{equation}
r_{1,i} =
-T^*\left[\mathbb{E}\left[\sum_{j=1}^n\sum_{k=1}^n \alpha_j^{(0)}\alpha_k^{(0)} K(X, x_j)K(X, x_k) \,\bigg|\, Z\right]\right](x_i)\,,
\end{equation}
where $Z=\cdot$ is selected to be the appropriate values of $Z$ for which $T^*$ reproduces the correct Hilbert space function.
By the properties of the adjoint operator $T^*$, this is equivalent to averaging over the function conditioned on all possible instrumental $Z$s, using properties of expectations, we can rearrange:
\begin{align}
& r_{1,i}
= -\sum_{j=1}^n\sum_{k=1}^n \alpha_j^{(0)}\alpha_k^{(0)} \times\\
&\mathbb{E}_Z\left[\mathbb{E}\left[K(X, x_i)K(X, x_j)K(X, x_k) \,\bigg|\, Z\right]\right]\,.
\end{align}
Due to the symmetries in the three-point interaction and the specific form of our cubic term, an additional combinatorial factor of 3 appears, giving us:
\begin{align}
&r_{1,i} = -3\times\\
&\sum_{j=1}^n\sum_{k=1}^n \alpha_j^{(0)}\alpha_k^{(0)}\mathbb{E}_Z\left[\mathbb{E}\left[K(X, x_i)K(X, x_j)K(X, x_k) \,\bigg|\, Z\right]\right] \nonumber\,.
\end{align}
Finally, defining the non-independent three-point kernel tensor:
\begin{align}
K_{ijk} = \mathbb{E}_Z\left[\mathbb{E}\left[K(X, x_i)K(X, x_j)K(X, x_k) \,\bigg|\, Z\right]\right]\,.
\end{align}
We arrive at:
\begin{align}
r_{1,i} &= -3\sum_{j=1}^n\sum_{k=1}^n \alpha_j^{(0)}\alpha_k^{(0)}K_{ijk}\,.
\end{align}
This matches exactly the computation in our algorithm.

\begin{theorem}[First-Order Representer Theorem]
The first-order correction $f_1$ has the representer form:
\begin{align}
f_1(x) = \sum_{i=1}^n \alpha_i^{(1)} K(x, x_i)\,,
\end{align}
where $\alpha^{(1)}$ is the solution to the linear system:
\begin{align}
(\widetilde{K} + \lambda K)\alpha^{(1)} = r_1\,,
\end{align}
with $r_{1,i} = -3\sum_{j=1}^n\sum_{k=1}^n \alpha_j^{(0)}\alpha_k^{(0)}K_{ijk}$.
\end{theorem}

\begin{proof}
Since the operator $(T^*T + \lambda I)$ maps the RKHS to itself while preserving the finite-dimensional subspace spanned by the data points, and the right-hand side $r_1$ lies in this subspace, the solution $f_1$ must also have the representer form.
\end{proof}
Essentially, the representer theorem does not incur any difficulty here either since the addition of anything orthogonal to the Hilbert space basis still projects to $0$, hence not improving the minization of the entire functional. 

\paragraph{Second-Order and Higher Corrections}

Following the same procedure, we can derive the second-order correction:
\begin{multline}
S[f_0 + \gamma f_1 + \gamma^2 f_2] = S_0[f_0] + \mathcal{O}(\gamma) +\\ \gamma^2\left(\frac{\delta S_0[f]}{\delta f}\bigg|_{f=f_0} \cdot f_2 + \frac{1}{2}\frac{\delta^2 S_0[f]}{\delta f^2}\bigg|_{f=f_0}(f_1, f_1) + \frac{\delta S_1[f]}{\delta f}\bigg|_{f=f_0} \cdot f_1\right) + O(\gamma^3)\,.
\end{multline}
Setting the variational derivative with respect to $f_2$ to zero, we have the constraint.
\begin{theorem}[Second-Order Representer Theorem]
The second-order correction $f_2$ also has the representer form:
\begin{align}
f_2(x) = \sum_{i=1}^n \alpha_i^{(2)} K(x, x_i)\,,
\end{align}
where $\alpha^{(2)}$ is the solution to the linear system:
\begin{align}
(\widetilde{K} + \lambda K)\alpha^{(2)} = r_2\,,
\end{align}
with $r_{2,i} = -3\sum_{j=1}^n\sum_{k=1}^n (\alpha_j^{(0)}\alpha_k^{(1)} + \alpha_j^{(1)}\alpha_k^{(0)})K_{ijk}$.
\end{theorem}

\paragraph{General Result for All Orders}
We can generalize these results to all orders in the perturbative expansion:
\begin{theorem}[Series Expansion Representer Theorem]
For the NPIV problem with a cubic interaction term, the solution can be expressed as a perturbative series:
\begin{align}
f^*(x) = \sum_{k=0}^{\infty} \gamma^k f_k(x) = \sum_{k=0}^{\infty} \gamma^k \sum_{i=1}^n \alpha_i^{(k)} K(x, x_i)\,,
\end{align}
where each $\alpha^{(k)}$ is determined by solving a linear system that depends on $\{\alpha^{(j)}\}_{j=0}^{k-1}$.
\end{theorem}
This is exactly what we have had in the algorithm.

%%%%%%%%%%%%%%%%%%%%%%%%%%%%%%%%%%%%%
\section{Interaction for rotationally invariant kernels}\label{appendix:new_int_term_breaking_rotationally_invariance}
The original triple moment interaction we added was:
\begin{equation}
K_{ijk} = \int K(x, x_i) K(x, x_j) K(x, x_k) f(x|Z)  dx\,.
\end{equation}

For rotationally invariant kernels $K(x,y) = k(\|x-y\|)$ with Mercer decomposition:
\begin{equation}
K(x,y) = \sum_{\alpha} \lambda_\alpha \phi_\alpha(x) \phi_\alpha(y)\,,
\end{equation}
where $\phi_\alpha(x) = \phi_{l,m}(x) = R_l(\|x\|) Y_l^m(x/\|x\|)$ with $Y_l^m(x/\|x\|)$ the usual spherical harmonics. 

As we have seen, these triple moments contribute to the final integral as the Clebsch-Gordan or Wigner 3-j symbols which has a strict selection rule for the labels, namely it is only non-zero if 
\begin{equation}
    m_{1} + m_2 + m_3 = 0, \quad l_1, l_2, l_3 \text{ even and form triangle}\,.
\end{equation}
This is rather sparse and zeros out most of the contributions. The most restrictive conditions are actually the delta function in $m$ quantum number and the fact that the $l$s satisfy the triangle inequality. We would like to explicitly overcome this, we simply introduce a factor that breaks rotational invariance:
\begin{equation}
K_{ijk}^{\text{aniso}} = \int K(x, x_i) K(x, x_j) K(x, x_k) [f(x|Z)]^3 W_A(x) dx\,,
\end{equation}

where the anisotropic weight is:
\begin{equation}
W_A(x) = \exp\left(-\frac{1}{2} x^T A^{-1} x\right)\,,
\end{equation}

with $A = \text{diag}(a_1, a_2, \ldots, a_n)$ and $a_i > 0$ all distinct.
With the anisotropic weight $W_A(x) = \exp(-\frac{1}{2} x^T A^{-1} x)$:
\begin{equation}
\mathcal{T}_{\alpha\beta\gamma}^{\text{aniso}} =
\int \phi_\alpha(x) \phi_\beta(x) \phi_\gamma(x) [f(x|Z)]^3 \exp\left(-\frac{1}{2} x^T A^{-1} x\right) dx\,.
\end{equation}
The anisotropic factor **breaks rotational symmetry**, making the integral generically non-zero even when the standard Clebsch-Gordan coefficient vanishes.

For $\phi_\alpha(x) = R_l(\|x\|) Y_l^m(x/\|x\|)$:
\begin{equation}
\mathcal{T}_{\alpha\beta\gamma}^{\text{aniso}} = \int_{\mathbb{R}^n} R_{l_1}(\|x\|) Y_{l_1}^{m_1}(x/\|x\|) R_{l_2}(\|x\|) Y_{l_2}^{m_2}(x/\|x\|) R_{l_3}(\|x\|) Y_{l_3}^{m_3}(x/\|x\|) [f(\|x\||Z)]^3 \exp\left(-\frac{1}{2} x^T A^{-1} x\right) dx\,.
\end{equation}

In spherical coordinates $x = r\hat{\omega}$:
\begin{equation}
\mathcal{T}_{\alpha\beta\gamma}^{\text{aniso}} = \int_0^{\infty} \int_{\mathbb{S}^{n-1}} R_{l_1}(r) R_{l_2}(r) R_{l_3}(r) [f(r|Z)]^3 
 Y_{l_1}^{m_1}(\hat{\omega}) Y_{l_2}^{m_2}(\hat{\omega}) Y_{l_3}^{m_3}(\hat{\omega})
\exp\left(-\frac{1}{2} r^2 \hat{\omega}^T A^{-1} \hat{\omega}\right) r^{n-1} dr d\hat{\omega}\,.
\end{equation}
The key angular integral is:
\begin{equation}\label{eq:key_angular}
\mathcal{A}_{l_1 m_1, l_2 m_2, l_3 m_3} = \int_{\mathbb{S}^{n-1}} Y_{l_1}^{m_1}(\hat{\omega}) Y_{l_2}^{m_2}(\hat{\omega}) Y_{l_3}^{m_3}(\hat{\omega}) 
\exp\left(-\frac{1}{2} r^2 \hat{\omega}^T A^{-1} \hat{\omega}\right) d\hat{\omega}\,.
\end{equation}
If $A^{-1} = \alpha I$ (isotropic), then:
\begin{equation}
\exp\left(-\frac{1}{2} r^2 \hat{\omega}^T A^{-1} \hat{\omega}\right) = \exp\left(-\frac{1}{2} \alpha r^2\right) = \text{const.}\,.
\end{equation}

This gives:
\begin{equation}
\mathcal{A}_{l_1 m_1, l_2 m_2, l_3 m_3} = \exp\left(-\frac{1}{2} \alpha r^2\right) \int_{\mathbb{S}^{n-1}} Y_{l_1}^{m_1}(\hat{\omega}) Y_{l_2}^{m_2}(\hat{\omega}) Y_{l_3}^{m_3}(\hat{\omega}) d\hat{\omega}\,.
\end{equation}

The integral is just the Clebsch-Gordan coefficient, subject to selection rules.
When $A^{-1} = \text{diag}(a_1^{-1}, a_2^{-1}, \ldots, a_n^{-1})$ with distinct values:
\begin{equation}
\exp\left(-\frac{1}{2} r^2 \hat{\omega}^T A^{-1} \hat{\omega}\right) = \exp\left(-\frac{1}{2} r^2 \sum_{j=1}^n a_j^{-1} \omega_j^2\right)\,,
\end{equation}
where $\omega_j$ are the components of $\hat{\omega}$ in Cartesian coordinates. We consider the key angular integral \eqref{eq:key_angular}. 

The integral is subject to a strict parity selection rule. The spherical harmonic has the property $Y_l^m(-\hat{\omega}) = (-1)^l Y_l^m(\hat{\omega})$. The exponential term is an even function with respect to the transformation $\hat{\omega} \to -\hat{\omega}$. The parity of the integrand is therefore determined solely by the product of the three spherical harmonics:
\[
\text{Integrand}(-\hat{\omega}) = (-1)^{l_1+l_2+l_3} \times \text{Integrand}(\hat{\omega})
\]
As the integral is over a symmetric domain, the integral of any odd function is zero.
\textbf{Conclusion:} The integral $I$ is zero unless the sum $l_1+l_2+l_3$ is an even integer.

Assuming $l_1+l_2+l_3$ is even, the integral can be computed by expanding the exponential term in a power series:
\begin{equation}
    \exp\left(-\frac{1}{2}\rho^2 \hat{\omega}^T A \hat{\omega}\right) = \sum_{k=0}^{\infty} \frac{1}{k!} \left(-\frac{\rho^2}{2}\right)^k (\hat{\omega}^T A \hat{\omega})^k\,.
\end{equation}
Substituting this into the integral and interchanging summation and integration yields a series representation for $I$:
\begin{equation}
    I = \sum_{k=0}^{\infty} \frac{(-\rho^2/2)^k}{k!} \mathcal{M}_k\,,
\end{equation}
where $\mathcal{M}_k$ are the moment integrals:
\begin{equation}
    \mathcal{M}_k = \int_{\mathbb{S}^{n-1}} Y_{l_1}^{m_1}(\hat{\omega}) Y_{l_2}^{m_2}(\hat{\omega}) Y_{l_3}^{m_3}(\hat{\omega}) \left(\sum_{i=1}^n \lambda_i \omega_i^2\right)^k d\hat{\omega}\,.
\end{equation}
These moments can be computed by expressing the integrand as a homogeneous polynomial in the Cartesian components of $\hat{\omega}$ and integrating term-by-term.

\paragraph{Selection Rules in Three Dimensions}
The presence of the anisotropic matrix $A$ breaks the continuous rotational symmetry of the sphere, which relaxes the strict selection rules associated with the conservation of angular momentum. To derive the new rules, we expand the anisotropic term $g(\hat{\omega}) = \exp(-\frac{1}{2}\rho^2 \hat{\omega}^T A \hat{\omega})$ in spherical harmonics:
\begin{equation}
    g(\hat{\omega}) = \sum_{L=0}^{\infty} \sum_{M=-L}^{L} C_{L,M} Y_L^M(\hat{\omega})\,.
\end{equation}
The integral $I$ can then be written as a sum over integrals of four harmonics:
\begin{equation}
    I = \sum_{L, M} C_{L,M} \int_{\mathbb{S}^2} Y_{l_1}^{m_1}(\hat{\omega}) Y_{l_2}^{m_2}(\hat{\omega}) Y_{l_3}^{m_3}(\hat{\omega}) Y_L^M(\hat{\omega}) d\Omega\,.
\end{equation}
The selection rules for $I$ are determined by the properties of the coefficients $C_{L,M}$ and the four-harmonic integrals.

The selection rules for the magnetic quantum numbers $m_i$ are determined by the azimuthal symmetry of $g(\hat{\omega})$. In 3D spherical coordinates, the argument of the exponential is:
\begin{align*}
    \hat{\omega}^T A \hat{\omega} &= \lambda_1 \sin^2\theta\cos^2\phi + \lambda_2 \sin^2\theta\sin^2\phi + \lambda_3 \cos^2\theta \\
    &= \left( \frac{\lambda_1+\lambda_2}{2}\sin^2\theta + \lambda_3\cos^2\theta \right) \\+ &\left( \frac{\lambda_1-\lambda_2}{2}\sin^2\theta \right) \cos(2\phi)\,.
\end{align*}
The dependence on the angle $\phi$ is entirely contained within a $\cos(2\phi)$ term. Any function of $\cos(2\phi)$ has a Fourier series containing only frequencies that are multiples of 2 (i.e., terms like $\cos(2k\phi)$). The expansion coefficients $C_{L,M} = \int g(\hat{\omega}) Y_L^{M*}(\hat{\omega}) d\Omega$ involve an integral over $\phi$ of the form:
\[
\int_0^{2\pi} F(\cos(2\phi)) e^{-iM\phi} d\phi\,.
\]
This integral is non-zero only if $M$ matches a frequency present in the function of $\cos(2\phi)$, which means \textbf{$M$ must be an even integer}.
The integral of four harmonics is non-zero only if the sum of the magnetic quantum numbers is zero: $m_1+m_2+m_3+M=0$. Since only terms with even $M$ contribute to the sum for $I$, we arrive at the relaxed selection rule:
\begin{equation}
m_1+m_2+m_3 = -M \implies  m_1+m_2+m_3 = \text{an even integer.}
\end{equation}

The triangle inequality for $(l_1, l_2, l_3)$ is lifted because the anisotropic field provides access to multiple angular momentum channels $L$. The integral of four harmonics is non-zero if $(l_1,l_2,l_3,L)$ can be coupled to a total angular momentum of zero. Since $g(\hat{\omega})$ is an even function, its expansion only contains even values of $L$.

\paragraph{Counterexample:} Consider the triplet $(l_1,l_2,l_3)=(1,5,1)$. This violates the standard triangle inequality $|l_1-l_2| \le l_3 \le l_1+l_2$ (since $1 \notin [4,6]$), so its Gaunt integral is zero. For the anisotropic integral, we check if a coupling pathway exists through an even $L$.
\begin{enumerate}
    \item Couple $(l_1,l_2)=(1,5) \implies$ intermediate angular momentum $l_{12} \in \{4,5,6\}$.
    \item We now seek an even $L$ that allows coupling $(l_3,L)=(1,L)$ to one of these $l_{12}$ values. Let's target $l_{12}=4$.
    \item The coupling rule requires $|l_3-L| \le l_{12} \le l_3+L$, which for our values becomes $|1-L| \le 4 \le 1+L$.
    \item The inequality $4 \le 1+L$ implies $L \ge 3$. The inequality $|1-L| \le 4$ implies $-3 \le L \le 5$.
\end{enumerate}
The conditions require $L$ to be in the range $[3, 5]$. We can choose the even value $L=4$. The expansion of $g(\hat{\omega})$ contains a non-zero $C_{4,M}$ term. Therefore, a coupling pathway exists via the $L=4$ channel, and the integral can be non-zero. This demonstrates that the triangle inequality on $(l_1, l_2, l_3)$ is lifted.

\paragraph{Summary:} Hence we see, we have managed to lift two of the most restrictive selection rule on the Clebsch-Gordan symbol, now with the new interaction term, it has non-zero contribution as long as the sum of both $m$ and $l$ are even, which is not such a restrictive condition. By no means is this choice unique, we could in principle choose any interaction term that integrates to object with little selection rule, hence improving the performance of the rotationally invariant kernels. 

\subsection{results}
Here we list some of the results we obtained in various dimensions, how this anisotropic term helps rescuing the performance of the Gaussian RBF kernel.

\begin{figure*}[ht]
    \centering
    \includegraphics[width=0.75\linewidth]{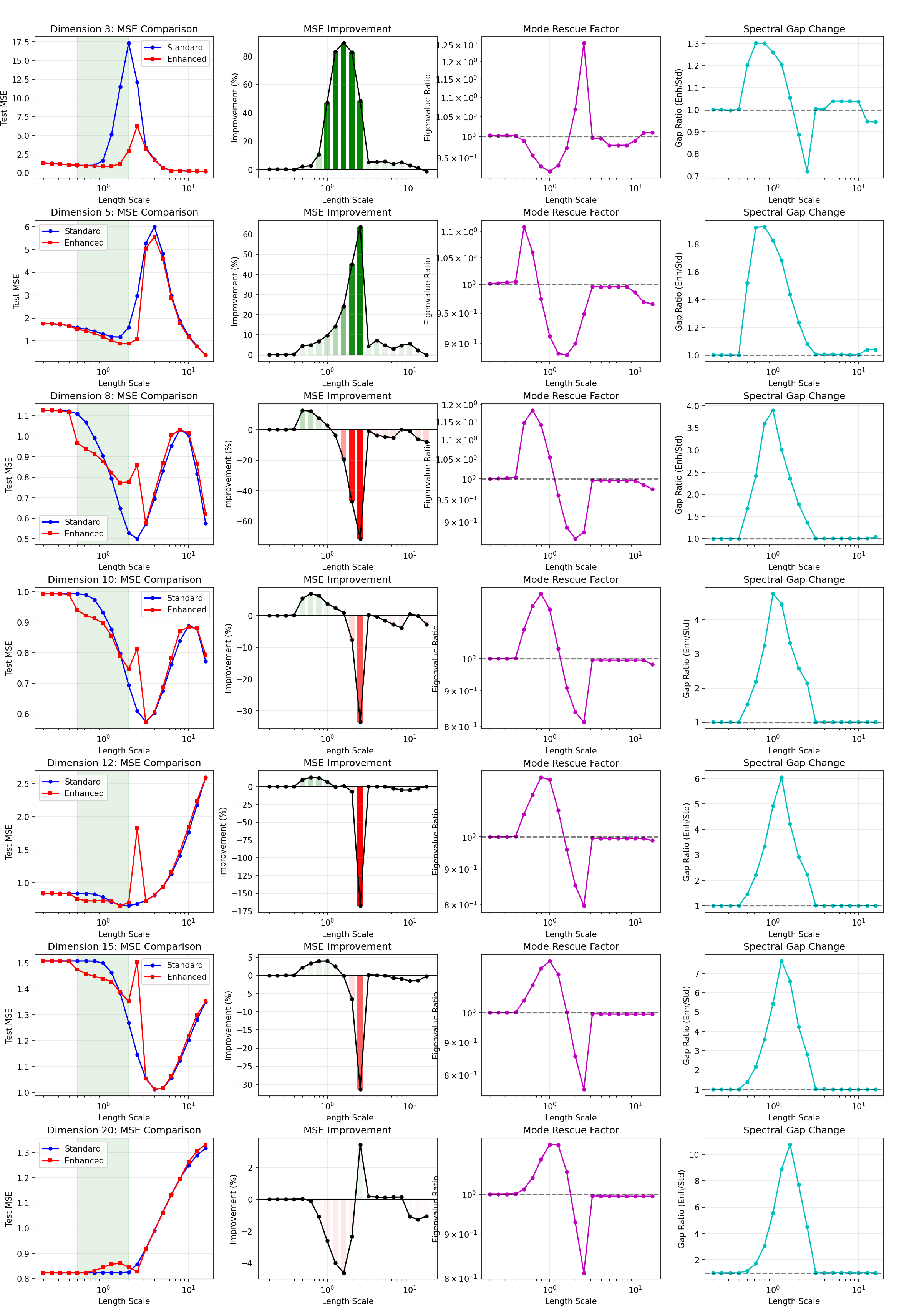}
    \caption{MSE comparison across dimension and RBF kernel bandwidth.}
    \label{fig:anisotropic}
\end{figure*}

As shown in figure \ref{fig:anisotropic}, it does help in lower dimension when kernels themselves do not vanish, even this depends on the length scale of the RBF kernel. For some choices the additional term even hurts the performance. Although there is a lift in coefficients, due to the vanishing of the kernels themselves in high dimensions, this still fails inevitably.

%%%%%%%%%%%%%%%%%%%%%%%%%%%%%%%%%%%%%%%%%%%%%%%%%%%

\section{Experiments}\label{appendix:experiments}

We introduce several datasets and summerize the performance of the perturbative method against kernel IV and Deep IV\footnote{All results on Deep IV was done using a two layered approximately 2k parameters neural net trained for 500 epochs for each regression stage.}. We note that all tests are done repeatedly 1000 times and the average RMSE was recorded here as the result for comparison.  

\subsection{Primary experiment dataset}\label{sec:dataset}
The primary dataset is constructed to match the experimental setup in \cite{donhauser2022fast}, with an instrumental variable twist. 

For a given sample size $n$ and dimensionality parameter $\beta$, we set $d = \lfloor n^\beta \rfloor$ and generate the following data. First the instrumental variable $Z_i \sim \mathcal{N}(0, I_d)$, then the confounding term with correlation across dimensions: $U_{X,i} \sim \mathcal{N}(0, I_d)$. Endogenous variables $X_i$ with nonlinear dependence on $Z$:
\begin{equation}
X_i = 0.7Z_i + 0.3Z_i^2 + 0.2U_{X,i}\,,
\end{equation}
where the operations are element-wise. Ground truth function combines nonlinear effects from multiple dimensions:
\begin{equation}
g(X_i) = 0.5\sin\left(\sum_{j=1}^{10} w_j X_{i,j}\right) + 0.5e^{-0.1 \sum_{j=11}^{20} w_j X_{i,j}^2}\,,
\end{equation}
where $w_j = 1/\sqrt{j}$ are dimension-specific weights. Effectively the useful feature dimensions are the first $20$ dimensions. Dimensions beyond $20$ are essentially noise dimensions, which creates a challenging scenario with high condition number for kernel ridge regression in NPIV, essentially kernel function needs to learn to ignore these noise.  
Finally we add error term with endogeneity:
\begin{equation}
\epsilon_{Y,i} = 0.8 \cdot \bar U_{X,i} + 0.2 \cdot \nu_i\,,
\end{equation}
where $\bar{U}_{X,i} = \frac{1}{d}\sum_{j=1}^{d} U_{X,i,j}$ is the average error across dimensions and $\nu_i \sim \mathcal{N}(0, 1)$ with finally outcome:
\begin{equation}
    Y_i = g(X_i) + \epsilon_{Y,i}\,.
\end{equation}
This construction creates a challenging high-dimensional regression problem where the true function depends nonlinearly on multiple dimensions, endogeneity is present through correlated errors and the dimensionality parameter $\beta$ controls how quickly dimension grows with sample size.

We tested across various dimensionality regimes by setting dimension $d \approx n^\beta$ for $\beta \in \{0.3, 0.5, 0.7, 0.9, 1.1\}$. This allowed us to observe how performance varies as the dimension increases relative to sample size. We tested the following parameters to identify optimal values for base perturbation parameter $\gamma \in \{0.4, 0.6, 0.8, 1.0\}$.

\subsection{Kernels}
Gaussian RBF kernels are simply:
\begin{equation}
    K_{\text{RBF}}(x, x') = \exp\left(-\frac{\|x-x'\|^2}{2\sigma^2}\right)\,,
\end{equation}
where $\sigma > 0$ is the bandwidth parameter which we test in logarithmic scale $0.01, 0.1, 1, 10$. It is rotationally invariant and has been proved to suffer from curse of dimensionality in \cite{donhauser2022fast}.

% polynomial kernels are defined as:
% \begin{equation}
%     K_{\text{poly}}(x, x') = (1+\langle x, x'\rangle)^d
% \end{equation}
% where $\langle\cdot, \cdot\rangle$ denotes the input space inner products and $d$ is the polynomial degree. We choose $d=2$ in our experiments. 

The Fractional Brownian kernel is based on fractional Brownian motion covariance functions. For points $x, x' \in \mathbb{R}^d$, the kernel reads:
\begin{equation}
    K_q(x, x') = \frac{1}{2}\left(\|x\|^q + \|x'\|^q - \|x - x'\|^q\right)\,,
\end{equation}
where $0 < q \leq 2$ is the fractional parameter. It explicitly breaks rotational invariance by choosing a reference point (here we chose the origin). This kernel belongs to the class of fractional Brownian kernels and corresponds to the covariance structure of fractional Brownian motion with Hurst parameter $H = q/2$. The kernel captures long-range dependence when $q > 1$ and exhibits different smoothness properties compared to standard RBF kernels \cite{Sejdinovic_2013, Rizzo-Sejdinovic2016}. In our implementation, we use $q = 1$ (corresponding to $H = 0.5$, standard Brownian motion). We shall refer to this kernel as the fractional Brownian kernel in the experiment section.

\subsection{high dimensional results}\label{appendix:detailed_results}
Here we include the performance of our algorithm across different kernels, dimensionalities and renormalization strength in table \ref{tab:order0_rmse} and \ref{tab:renormalized_rmse}.

% Table 1: Order 0 RMSE (Baseline) - All Kernels
\begin{table*}[htbp]
\centering
\caption{Order 0 RMSE (Baseline) Across All Kernels and Deep IV}
\label{tab:order0_rmse}
\begin{tabular}{c|cccc|cccc|cccc|c}
\toprule
\multirow{1}{*}{$\beta$} & \multicolumn{1}{c|}{\textbf{RBF}} & \multicolumn{1}{c}{\textbf{fractional Brownian}} & \multicolumn{1}{c}{\textbf{Deep IV}} \\
\midrule
0.3 & 0.317  & 0.235 &  0.121\\
0.5 & 0.536  & 0.173 & 0.187 \\
0.7 & 0.289  & 0.382 & 0.225 \\
1.0 & 0.360  & 1.654 & 0.089 \\
1.3 & 0.511  & 3.766 & 0.078 \\
1.5 & 0.467  & 11.51 & 0.096 \\
\bottomrule
\end{tabular}
\end{table*}

% Table 2: Renormalized RMSE - All Kernels
\begin{table*}[htbp]
\centering
\caption{Renormalized RMSE Across All Kernels}
\label{tab:renormalized_rmse}
\begin{tabular}{c|cccc|cccc}
\toprule
\multirow{2}{*}{$\beta$} & \multicolumn{4}{c|}{\textbf{RBF}} & \multicolumn{4}{c}{\textbf{Frac. Brownian}} \\
& \multicolumn{4}{c|}{$\gamma$} & \multicolumn{4}{c}{$\gamma$} \\
& 0.4 & 0.6 & 0.8 & 1.0 & 0.4 & 0.6 & 0.8 & 1.0 \\
\midrule
0.3 & 0.318 & 0.318 & 0.319 & 0.319 & 0.238 & 0.240 & 0.241 & 0.243 \\
0.5 & 0.536 & 0.536  & 0.536  & 0.536 & 0.182 & 0.187 & 0.193 & 0.992 \\
0.7 & 0.289 & 0.289 & 0.289 & 0.289 & 0.233 & 0.181 & 0.145 & 0.124 \\
1.0 & 0.360 & 0.360 & 0.360 & 0.360 & 0.322 & 0.108 & 0.194 & 0.556 \\
1.3 & 0.511 & 0.511 & 0.511 & 0.511 & 0.513 & 0.028 & 0.327 & 1.370 \\
1.5 &  0.465 & 0.465 & 0.465 & 0.465 & 1.935 & 0.136 & 2.274 & 10.99 \\
\bottomrule
\end{tabular}
\end{table*}

\subsection{Further datasets results}\label{subappendix:further_datasets}
We evaluate our methods on several alternative NPIV datasets beyond the one from \cite{donhauser2022fast} described in appendix \ref{sec:dataset}. Each dataset is generated with $n=80$ samples with results summerized in table \ref{tab:alternative_datasets_rmse2}.

\paragraph{Newey-Powell Dataset.} A classic NPIV benchmark following \cite{newey2003instrumental}:
\begin{align}
Z_i &\sim \text{Uniform}(-3, 3), \quad V_i \sim \mathcal{N}(0, 1) \\
X_i &= 0.6 Z_i + 0.4 V_i + 0.2 \nu_i, \quad \nu_i \sim \mathcal{N}(0, 1) \\
g(X) &= \sin(X) + 0.5 X \\
Y_i &= g(X_i) + 0.5 V_i + 0.5 \epsilon_i, \quad \epsilon_i \sim \mathcal{N}(0, 1)
\end{align}

\paragraph{Weak/Strong Instrument Dataset.} Tests robustness to instrument strength $\rho \in \{0.1, 0.3, 0.7\}$ with $d=5$:
\begin{align}
Z_i, U_i &\sim \mathcal{N}(0, I_d) \\
X_i &= \rho Z_i + (1-\rho) U_i + 0.1 \nu_i, \quad \nu_i \sim \mathcal{N}(0, I_d) \\
g(X) &= 0.5 \sin\left(\sum_{j=1}^{d} w_j X_j\right) + 0.3 \cos\left(\sum_{j=1}^{2} w_j X_j\right) \\
Y_i &= g(X_i) + 0.6 \bar{U}_i + 0.4 \epsilon_i
\end{align}
where $w_j = 1/\sqrt{j}$ and $\bar{U}_i = d^{-1}\sum_{j=1}^d U_{ij}$.

\paragraph{Heteroscedastic Dataset.} Features variance that depends on $X$ with $d=10$:
\begin{align}
Z_i, U_i &\sim \mathcal{N}(0, I_d), \quad X_i = 0.7 Z_i + 0.3 U_i \\
g(X) &= \sin\left(\sum_{j=1}^{5} w_j X_j\right) + 0.3 \exp\left(-0.1 \sum_{j=6}^{d} X_j^2\right) \\
\sigma^2(X) &= 0.5 + 0.5 |\bar{X}|, \quad Y_i = g(X_i) + 0.5 \bar{U}_i + \sigma(X_i) \epsilon_i
\end{align}

\paragraph{Nonlinear Instrument Dataset.} Tests nonlinear first-stage relationships with $d=10$:
\begin{align}
Z_i, U_i &\sim \mathcal{N}(0, I_d) \\
X_i &= 0.5 Z_i + 0.3 Z_i^2 + 0.2 \sin(Z_i) + 0.2 U_i \\
g(X) &= 0.4 \tanh\left(\sum_{j=1}^{5} w_j X_j\right) + 0.3 \cos\left(\sum_{j=6}^{10} w_j X_j\right) \\
Y_i &= g(X_i) + 0.6 \bar{U}_i + 0.4 \epsilon_i
\end{align}

\paragraph{Sparse Signal Dataset.} High-dimensional setting ($d=50$) with only $s=5$ active dimensions:
\begin{align}
Z_i, U_i &\sim \mathcal{N}(0, I_d), \quad X_i = 0.7 Z_i + 0.3 U_i \\
g(X) &= \sin\left(\frac{1}{\sqrt{s}}\sum_{j=1}^{s} X_j\right) + 0.3 \cos(2 X_1) \\
Y_i &= g(X_i) + 0.5 \bar{U}_{1:s,i} + 0.5 \epsilon_i
\end{align}
where $\bar{U}_{1:s,i} = s^{-1}\sum_{j=1}^s U_{ij}$ averages only the active dimensions.

\begin{table*}[htbp]
\centering
\caption{RMSE Comparison on Alternative NPIV Datasets}
\label{tab:alternative_datasets_rmse2}
\begin{tabular}{l|ccc}
\toprule
\textbf{Dataset} & \textbf{DeepIV} & \textbf{Order 0} & \textbf{Best Pert.} \\
\midrule
Newey-Powell & \textbf{1.130} & 14.858 & 8.718  \\
Weak IV ($\rho=0.1$) & 0.347 & 2.779 & \textbf{0.280}  \\
Weak IV ($\rho=0.3$) & \textbf{0.289} & 2.165 & 0.339 \\
Strong IV ($\rho=0.7$) & \textbf{0.255} & 2.175 & 0.329 \\
Heteroscedastic & \textbf{0.678} & 7.969 & 0.724  \\
Nonlinear IV & \textbf{0.388} & 14.986 & 0.409  \\
Sparse ($d=50$) & 0.586 & 13.833 & \textbf{0.560}  \\
\bottomrule
\end{tabular}
\footnotesize
\begin{tablenotes}
\item RMSE = $\sqrt{\text{MSE}}$; ``Order 0'' is standard kernel ridge IV; ``Best Pert.'' is the best result across $\gamma \in \{0.4, 0.6, 0.8, 0.99\}$ and kernels $\in$ \{FB, RBF\}. Bold indicates best method per dataset.
\end{tablenotes}
\end{table*}

In table \ref{tab:alternative_datasets_rmse2}, we note the improvement of the perturbative method on order $0$ base kernel IV across all these datasets. Although DeepIV performs well on most datasets, the perturbative method shows competitive or superior performance in two notable cases: (1) the weak instrument setting ($\rho=0.1$), where both neural network training and kernel IV struggle but the perturbative correction provides meaningful improvement; and (2) the sparse high-dimensional setting ($d=50$), where the perturbative approach's ability to enhance kernel expressivity in specific eigendirections proves advantageous. The Order 0 baseline (standard kernel ridge regression) performs poorly across all datasets, highlighting the importance of addressing the ill-posedness inherent in NPIV problems.

\subsection{Weak Instrument Analysis}\label{sec:weak_instruments}

A fundamental challenge in instrumental variable estimation is the \textit{weak instrument problem}, which occurs when instruments $Z$ have low correlation with the endogenous variable $X$. Weak instruments lead to biased estimates, inflated standard errors, and unreliable inference \citep{Bound1995, Staiger1997}. The first-stage $R^2$---measuring the proportion of variance in $X$ explained by $Z$---serves as a diagnostic: values below 0.1 are conventionally considered indicative of weak instruments \citep{Stock01102002}.

We systematically vary the instrument strength parameter $\rho \in [0.05, 0.9]$ in our weak instrument dataset (Section~\ref{subappendix:further_datasets}) to examine how different NPIV methods degrade as instruments weaken. The first-stage relationship is $X = \rho Z + (1-\rho) U + \text{noise}$, so $\rho$ directly controls instrument relevance. We summerize the results in table \ref{tab:weak_iv_rmse}.

\begin{table*}[htbp]
\centering
\caption{RMSE Across Instrument Strength Levels}
\label{tab:weak_iv_rmse}
\begin{tabular}{cc|ccc}
\toprule
\textbf{Strength} $\rho$ & \textbf{1st Stage} $R^2$ & \textbf{DeepIV} & \textbf{Order 0} & \textbf{Best Pert.} \\
\midrule
0.05 & 0.076 & \textbf{0.709} & 2.450 & 2.250  \\
0.10 & 0.072 & \textbf{0.882} & 2.265 & 1.999  \\
0.20 & 0.086 & \textbf{0.551} & 1.944 & 1.644 \\
0.30 & 0.144 & \textbf{0.470} & 1.732 & 1.463  \\
0.50 & 0.439 & \textbf{0.390} & 1.762 & 1.506  \\
0.70 & 0.810 & \textbf{0.359} & 2.446 & 1.826  \\
0.90 & 0.974 & \textbf{0.312} & 3.742 & 1.491 \\
\bottomrule
\end{tabular}
\footnotesize
\begin{tablenotes}
\item RMSE = $\sqrt{\text{MSE}}$. First-stage $R^2 < 0.1$ indicates weak instruments (shaded region $\rho \leq 0.2$). Bold indicates best method.
\end{tablenotes}
\end{table*}

As shown in table \ref{tab:weak_iv_rmse}, perhaps unsurprisingly, DeepIV dominates across all instrument strengths. The neural network's flexible function approximation proves advantageous even with weak instruments, though its RMSE increases from 0.312 ($\rho=0.9$) to 0.882 ($\rho=0.1$)---a 2.8$\times$ degradation. Compared to order $0$ kernel IV with regularization, perturbative corrections helps to reduce Order 0 RMSE by 18--37\% across settings, demonstrating the value of higher-order corrections even in challenging regimes. However, it cannot match DeepIV's performance.

\subsection{Sensitivity Analysis}\label{sec:sensitivity}

The perturbative renormalization framework introduces several hyperparameters. We conduct a systematic sensitivity analysis on the Donhauser dataset with $\beta=1.0$ ($d=60$, $n=60$) to understand their effects and identify robust default values.

We study parameters across the range, base perturbation coupling constant $\gamma \in \{0.1, 0.2, \ldots, 0.9, 0.95, 0.99\}$, Tikhonov regularization strength $\lambda \in \{0.01, 0.05, 0.1, 0.3, 0.6, 1.0, 2.0\}$, maximum perturbative order $N_{\max} \in \{2, 3, 4, 5, 6, 7\}$, summerized in table \ref{tab:sensitivity_summary}.

\begin{table*}[htbp]
\centering
\caption{Sensitivity Analysis: Optimal Parameters and RMSE Ranges}
\label{tab:sensitivity_summary}
\begin{tabular}{l|c|cc|c}
\toprule
\textbf{Parameter} & \textbf{Optimal Value} & \textbf{Best RMSE} & \textbf{Worst RMSE} & \textbf{Sensitivity} \\
\midrule
$\gamma$ (coupling) & 0.6 & 0.329 & 0.847 & High \\
$\lambda$ (regularization) & 0.6 & 0.329 & 0.512 & Medium \\
$N_{\max}$ (max order) & 5 & 0.329 & 0.398 & Low \\
\bottomrule
\end{tabular}
\footnotesize
\begin{tablenotes}
\item Sensitivity rated as: High ($>$50\% RMSE variation), Medium (20--50\%), Low ($<$20\%).
\end{tablenotes}
\end{table*}

As shown in table \ref{tab:sensitivity_summary}, the coupling constant $\gamma$ is the most sensitive parameter, as is confirmed in the result in table \ref{tab:renormalized_rmse}. The optimal value $\gamma^* \approx 0.6$ reflects a bias-variance tradeoff: small $\gamma$ underweights higher-order corrections (high bias), while $\gamma \to 1$ amplifies potentially divergent terms (high variance).

Moderate sensitivity with Regularization $\lambda$ with optimal $\lambda^* = 0.6$. Too little regularization ($\lambda < 0.1$) allows ill-conditioned kernel inversions; too much ($\lambda > 1$) over-smooths the solution.

Performance is stable for $N_{\max} \geq 4$, indicating that most information is captured in the first few orders. This aligns with asymptotic theory: optimal truncation typically occurs at $N^* = O(\gamma^{-1})$.

\paragraph{Recommended Defaults.} Based on this analysis, we recommend: $\gamma = 0.6$, $\lambda = 0.6$, $N_{\max} = 5$. These values perform well across the datasets in Sections~\ref{sec:experiments}--\ref{subappendix:further_datasets} without dataset-specific tuning.

\subsection{Larger sample experiments}\label{subappendix:largersampleexperiments}
Here we include a further set of experiments with larger sample size $(n=120)$ instead $n=80$. We summerize the results here in table \ref{tab:large_dataset_rmse}.

\begin{table*}[htbp]
\centering
\caption{RMSE Comparison on Alternative NPIV Datasets}
\label{tab:large_dataset_rmse}
\begin{tabular}{l|ccc}
\toprule
\textbf{Dataset} & \textbf{DeepIV} & \textbf{Order 0} & \textbf{Best Pert.} \\
\midrule
Newey-Powell & \textbf{1.388} & 3.444 & 2.129  \\
Heteroscedastic & \textbf{0.599}  & 5.909 & 0.716  \\
Nonlinear IV & 0.417 & 16.721 &\textbf{0.407} \\
Sparse ($d=50$) & \textbf{0.551} & 6.526 & 0.642  \\
\bottomrule
\end{tabular}
\footnotesize
\begin{tablenotes}
\item RMSE = $\sqrt{\text{MSE}}$; ``Order 0'' is standard kernel ridge IV; ``Best Pert.'' is the best result across $\gamma \in \{0.4, 0.6, 0.8, 0.99\}$ and kernels $\in$ \{FB, RBF\}. Bold indicates best method per dataset.
\end{tablenotes}
\end{table*}
%%%%%%%%%%%%%%%%%%%%%%%%%%%%%%%%%%%%%%%%%%%%%%%%%%%%%%%%%%%%%%%%%%%%%%%%%%%%%%%
%%%%%%%%%%%%%%%%%%%%%%%%%%%%%%%%%%%%%%%%%%%%%%%%%%%%%%%%%%%%%%%%%%%%%%%%%%%%%%%

\end{document}